\documentclass[12pt]{article}

\usepackage{authblk}
\usepackage{amsmath}
\usepackage{graphicx}
\usepackage{enumerate}
\usepackage{natbib}
\usepackage{url}
\usepackage{tcolorbox}

\usepackage{enumitem}
\usepackage{amsmath}    
\usepackage{amssymb}    
\usepackage{amsthm}     

\theoremstyle{plain}
\newtheorem{theorem}{Theorem}

\newtheorem{lemma}[theorem]{Lemma}

\theoremstyle{definition}

\usepackage{caption}
\usepackage{graphicx}
\usepackage{xcolor}
\usepackage{algorithm}
\usepackage{algorithmic}
\usepackage{datetime2} 

\usepackage{tikz}
\usetikzlibrary{positioning,arrows.meta}
\usepackage{booktabs}  

\newenvironment{narrow}[2]{%
  \begin{list}{}{%
      \setlength{\topsep}{0pt}%
      \setlength{\leftmargin}{#1}%
      \setlength{\rightmargin}{#2}%
      \setlength{\listparindent}{\parindent}%
      \setlength{\itemindent}{\parindent}%
      \setlength{\parsep}{\parskip}}%
    \item[]}{\end{list}
}

\begin{document}

\setlength{\parskip}{0.4\baselineskip plus 2pt}
\setlength{\parindent}{0pt}

\def\spacingset#1{\renewcommand{\baselinestretch}%
{#1}\small\normalsize} \spacingset{1}

\title{\bf The Geometry of Machine Learning Models\thanks{Research reported in this publication was supported by grant number INV-048956 from the Gates Foundation.}}
\author{Pawel Gajer$^{1}$\thanks{Corresponding author: pgajer@gmail.com}\; and Jacques Ravel$^{1}$}
\date{%
\small
$^1$Center for Advanced Microbiome Research and Innovation (CAMRI),\\
Institute for Genome Sciences, and Department of Microbiology and Immunology,\\
University of Maryland School of Medicine
}
\maketitle

\begin{abstract}
  This paper presents a mathematical framework for analyzing machine learning
  models through the geometry of their induced partitions. By representing
  partitions as Riemannian simplicial complexes, we capture not only adjacency
  relationships but also geometric properties including cell volumes, volumes of
  faces where cells meet, and dihedral angles between adjacent cells. For neural
  networks, we introduce a differential forms approach that tracks geometric
  structure through layers via pullback operations, making computations
  tractable by focusing on data-containing cells. The framework enables
  geometric regularization that directly penalizes problematic spatial
  configurations and provides new tools for model refinement through extended
  Laplacians and simplicial splines. We also explore how data distribution
  induces effective geometric curvature in model partitions, developing discrete
  curvature measures for vertices that quantify local geometric complexity and
  statistical Ricci curvature for edges that captures pairwise relationships
  between cells. While focused on mathematical foundations, this geometric
  perspective offers new approaches to model interpretation, regularization, and
  diagnostic tools for understanding learning dynamics.
\end{abstract}


\section{Introduction}
Modern machine learning models achieve remarkable empirical success, yet their
internal mechanisms remain largely opaque. While theoretical foundations have
evolved from early function approximation results
\cite{cybenko1989approximation, hornik1989multilayer} to include geometric
perspectives such as neural tangent kernels \cite{jacot2018neural}, loss
landscape analysis \cite{li2018visualizing}, and geometric deep learning
\cite{bronstein2017geometric}, these approaches have not systematically
characterized how models partition their input spaces and the geometric
properties of these partitions. Similarly, while interpretability methods have
progressed beyond pure input-output analysis to include feature visualization
\cite{olah2017feature} and representation probing \cite{alain2017understanding,
  belinkov2017analysis}, they typically do not examine the spatial organization
of decision regions—how cells meet, their relative volumes, and the angles at
which boundaries intersect. Performance metrics provide precise measurements of
accuracy but incomplete explanations for why certain architectural and training
choices lead to better generalization \cite{zhang2017understanding,
  neyshabur2017exploring}. This opacity becomes particularly problematic when
deploying models in critical applications where understanding failure modes is
paramount \cite{rudin2019stop, doshi2017towards}.

This paper develops a framework that analyzes machine learning models through
their inherent geometric structures. Many successful methods, including decision
trees \cite{breiman1984classification}, multivariate adaptive regression splines
\cite{friedman1991multivariate}, random forests \cite{breiman2001random},
gradient boosting \cite{friedman2001greedy}, and neural networks
\cite{lecun2015deep, goodfellow2016deep}, fundamentally operate by partitioning
their input space into regions. These partitions possess geometric properties
including cell volumes, volumes of faces where cells meet, and dihedral angles
between adjacent cells that provide a new lens for understanding model behavior.

The mathematical framework we develop represents machine learning models as
simplicial complexes, establishing a geometric interpretation that applies
across diverse model classes. Simplicial complexes provide a natural way to
encode not just which regions are adjacent (as in a graph), but how multiple
regions meet simultaneously, information that graphs cannot capture. When four
decision tree cells meet at a point, or when three neural network regions share
a common boundary, these higher-order relationships contain geometric
information essential for understanding model behavior.

Through the nerve construction \cite{borsuk1948theorem}, any partition
transforms into a simplicial complex where vertices represent regions, edges
capture adjacency, and higher-dimensional simplices encode multi-way
intersections. Decision trees create axis-aligned partitions whose simplicial
complexes have a characteristic structure: each point where $2^n$ cells meet (in
$n$ dimensions) corresponds to a $(2^n-1)$-dimensional simplex in the nerve
complex. Ensemble methods generate overlapping partitions that must be carefully
reconciled, with each base model contributing its own partition structure. The
reconciled partition then induces a single simplicial complex that captures the
ensemble's combined geometry. For neural networks, we analyze their structure
through refined partitions: as data flows through layers, we track how cells
subdivide and map between layers, constructing functionally enriched simplicial
complexes that capture both the combinatorial adjacency relationships and the
affine transformations within each region.

The framework equips these simplicial complexes with a Riemannian structure
through inner products on chain spaces that encode the geometric properties of
partitions: cell volumes, face volumes, and dihedral angles in a mathematically
precise way. For neural networks specifically, we employ a differential forms
approach where volumes and boundary measures transform systematically under
pullback operations through the layers, yielding explicit formulas that reduce
geometric computations from potentially exponentially many cell intersections to
calculations on only cells that actually contain training data. This provides a
geometric characterization of deep networks that goes beyond simply noting their
piecewise linear structure.

Our geometric framework complements but differs from existing approaches like
persistent homology \cite{carlsson2009topology, edelsbrunner2008persistent},
which studies topological features of data, and manifold learning
\cite{tenenbaum2000global, roweis2000nonlinear}, which seeks low-dimensional
representations. Rather than analyzing data geometry, we study the geometry of
model-induced partitions, revealing how learning algorithms organize and
navigate feature space.

The geometric framework enables regularization strategies that directly target
spatial properties associated with good generalization. While traditional
regularization operates indirectly by limiting tree depth or penalizing weight
magnitudes \cite{tibshirani1996regression, zou2005regularization}, geometric
regularization directly penalizes problematic spatial configurations: extreme
volume disparities between regions, sharp angles at boundaries, or excessive
fragmentation.

Beyond regularization, the geometric framework enables new approaches to model
refinement. We introduce extended Laplacians that incorporate higher-order
simplicial information, moving beyond traditional graph-based smoothing to
respect the full geometric structure of how cells meet in space. We also develop
simplicial splines that extend classical spline regression to high-dimensional
data by leveraging the Riemannian simplicial complex structure, avoiding the
limitations of tensor product approaches while respecting the partition
geometry. These tools offer new ways to smooth and refine the piecewise constant
predictions of partition models while respecting their inherent geometric
structure.

This geometric perspective also suggests new directions for understanding
machine learning. We introduce density-weighted metrics on the simplicial
complexes that incorporate the distribution of training data, making paths
through high-density regions effectively shorter than those through sparse
regions. This enriched geometric structure captures how data concentration
affects model behavior and provides new tools for analyzing generalization.

A key insight of our framework is that real-world data, despite living in
high-dimensional Euclidean space, is typically sparse and highly structured.
This sparsity and non-uniform distribution effectively creates geometric
curvature. When data concentrates in specific regions or along lower-dimensional
structures while leaving vast regions empty, machine learning models must adapt
their partitions to this inhomogeneous structure. The resulting geometry
exhibits curvature-like properties that our framework can measure and quantify.

This observation leads to concepts that extend classical differential geometry
to the real-world data setting. We develop statistical curvature measures that
quantify local geometric complexity through multiple approaches: ball-growth
curvature that measures how density-weighted neighborhoods expand, geodesic
spreading that tracks path divergence, and functional curvatures that capture
how model predictions vary across the partition. Building on these vertex-based
measures, we introduce statistical Ricci curvature for edges that combines
geometric, density, and functional components. This edge-based curvature
captures relational properties between cells that vertex measures alone cannot
detect, enabling a more complete geometric characterization of model behavior.

The curvature framework provides practical tools for model analysis and
improvement. We develop geometric regularization that penalizes extreme
curvatures at both vertices and edges, encouraging models with balanced
geometric structure. For monitoring learning dynamics, we track how curvature
distributions evolve during training, potentially revealing transitions from
learning robust patterns to overfitting. These geometric diagnostics complement
traditional loss-based metrics, offering new insights into model behavior.

While this paper focuses on developing the mathematical foundations, the
geometric framework opens several avenues for practical impact. The geometric
regularization methods could lead to models with better interpretability and
more predictable failure modes. The ability to quantify geometric properties
like cell volumes and boundary angles provides new diagnostic tools for
understanding when and why models overfit. For practitioners, the framework
offers a principled approach to comparing architectures based on their induced
geometric structures rather than solely on performance metrics. Experimental
validation of these theoretical insights remains an important direction for
future work.

The rest of the paper is organized as follows. Section 2 introduces the
geometric framework through partition models, where the connection between model
structure and induced geometry appears most directly. We begin with decision
trees and their natural simplicial complexes, extend to ensemble methods that
create overlapping partitions, and develop the theory of geometric
regularization. Section 2.4 explores the incorporation of statistical structure
through density-weighted metrics and the emergence of curvature-like phenomena,
developing both vertex-based and edge-based curvature measures along with their
applications to regularization and model monitoring. Section 3 extends the
framework to neural networks, showing how layer-wise computation induces
sequences of simplicial complexes with rich geometric structure.

\section*{2. Partition Models}

Partition-based machine learning models, from decision trees to their powerful
ensemble extensions, have become fundamental tools in modern data analysis.
These models succeed by dividing the feature space into regions, each associated
with a local prediction. Yet despite their widespread use and empirical success,
we have limited tools for understanding and improving the geometric structures
these partitions create.

Consider the spatial patterns that emerge when a decision tree recursively
splits a feature space. Some regions may become extremely small to accommodate
individual training points, while others span large volumes. Boundaries between
regions follow axis-aligned or linear cuts that may create jagged, indirect
paths between naturally connected areas. These geometric properties, including
cell volumes, boundary configurations, and angles at which regions meet, contain
rich information about model behavior that current methods largely ignore.

What if we could systematically analyze these geometric structures? Rather than
treating partitions merely as computational byproducts, we can view them as
geometric objects amenable to mathematical analysis. This perspective opens new
avenues: detecting when a model creates unnecessarily complex spatial
structures, identifying opportunities for smoothing without sacrificing
predictive power, and developing principled approaches to combine multiple
partitions in ensembles.

This section develops a geometric framework for understanding partition-based
models through the lens of Riemannian simplicial complexes. We begin by
reviewing how various machine learning models create partitions and their
current regularization strategies (Section 2.1). We then introduce a geometric
perspective that captures not just which cells are adjacent but how they meet in
space (Section 2.2). This framework enables novel analysis tools and
regularization approaches that directly leverage spatial structure. Finally, we
demonstrate how this geometric understanding leads to principled model
refinement techniques using simplicial splines and extended Laplacians (Section
2.3).

\subsection*{2.1 Partition Models in Machine Learning}

How do machine learning models make sense of complex, high-dimensional data? One
of the most successful strategies involves dividing the feature space into
manageable regions, with each region receiving its own tailored prediction. A
decision tree examining house prices might separate urban from suburban
properties, then further divide based on square footage, creating a hierarchy of
increasingly specific regions. Each leaf of the tree corresponds to a subset of
the feature space where houses share similar characteristics and, presumably,
similar prices.

This divide-and-conquer approach appears throughout machine learning. Random
forests combine hundreds of such partitions to achieve remarkable predictive
power. Multivariate adaptive regression splines create smooth transitions
between regions while maintaining the essential piecewise structure. Support
vector machines with local models adapt their decision boundaries to regional
data characteristics. What unites these diverse methods is their fundamental
operation: partitioning the feature space and assigning local models to each
piece.

Yet despite their empirical success, we lack a unified framework for
understanding these partition-based models. What geometric principles govern
effective partitions? When a model creates hundreds of tiny regions to fit
training data, how can we distinguish necessary complexity from overfitting? How
should adjacent regions relate to each other geometrically? These questions
motivate a systematic study of partition models through the lens of geometry.

To establish this geometric perspective, we first formalize the notion of a
partition model. Given observations $x_1, x_2, \ldots, x_N \in \mathbb{R}^n$ and
responses $y_1, y_2, \ldots, y_N \in \mathbb{R}$, a partition model constructs a
decomposition of the feature space into regions, each associated with a local
prediction function. Formally, the model defines a partition
$\mathcal{P} = \{C_\alpha\}_{\alpha}$ of a domain
$\mathcal{D} \subset \mathbb{R}^n$, where $\mathcal{D}$ is typically chosen as
the smallest axis-aligned hyperrectangle containing all observed data points.
The cells $C_\alpha$ satisfy the covering and disjointness properties:
$$
\bigcup_{\alpha} C_\alpha = \mathcal{D}, \quad C^\circ_\alpha \cap C^\circ_{\alpha'} = \emptyset \text{ for } \alpha \neq \alpha',
$$
where $C^\circ$ denotes the interior of cell $C$. Let $I$ denote the set of
indices $\alpha$ such that cell $C_\alpha$ contains at least one observation
from $\{x_1, x_2, \ldots, x_N\}$. While the partition may generate numerous
cells covering the entire domain $\mathcal{D}$, the model considers only those
cells indexed by $I$. This distinction is crucial: a decision tree might
partition $\mathcal{D}$ into thousands of regions, but only the subset
containing training data influences the model's predictions and parameters. Each
cell $C_\alpha$ with $\alpha \in I$ is associated with a local prediction
function $f_\alpha: C_\alpha \to \mathbb{R}$. These functions are determined by
minimizing an objective function that typically includes a data fidelity term
and may incorporate regularization:
$$
\mathcal{L} = \sum_{\alpha \in I} \left( \sum_{j} \|f_\alpha(x_{\alpha,j}) - y_{\alpha,j}\|^2 + R_\alpha \right),
$$
where $x_{\alpha,j}$ are the data points falling within $C_\alpha$,
$y_{\alpha,j}$ are their associated responses, and $R_\alpha$ represents an
optional regularization term. When present, $R_\alpha$ may reflect the
complexity of $f_\alpha$, geometric properties of the cell, or structural
constraints of the model class. Many partition models, particularly those with
inherent structural constraints such as constant predictions within cells, may
not require explicit regularization terms, setting $R_\alpha = 0$.

This formal framework encompasses a rich variety of machine learning methods,
each making different choices about how to create partitions and what functions
to use within each cell. Understanding these choices and their geometric
implications forms the foundation for our subsequent analysis.

\subsubsection*{2.1.1 Common Partition Models}

How do different machine learning algorithms create their partitions, and what
geometric structures result from their choices? When we examine a trained
decision tree, we see axis-aligned boxes dividing the feature space. A nearest
neighbor classifier implicitly creates Voronoi cells around training points. A
multivariate adaptive regression spline maintains smooth transitions across
piecewise boundaries. These diverse geometric structures arise from
fundamentally different partition strategies, yet we lack a unified way to
compare and analyze them.

Understanding these geometric differences matters for practical model selection.
What determines the magnitude of improvement when moving from a single decision
tree to a random forest? How do the axis-aligned partitions of trees compare to
the radial partitions of RBF networks in capturing different data patterns? By
examining the partition structures created by various models, we can begin to
understand their inductive biases and predict when each approach will excel.

The most familiar examples of partition models are tree-based methods
\cite{breiman1984classification,hastie2009elements}. In regression and
classification trees, the domain $\mathcal{D}$ is recursively divided through
axis-aligned splits of the form $x^{(j)} \leq t$. Each split introduces a
hyperplane $x^{(j)} = t$ that partitions the current region into two subregions,
ultimately producing hyperrectangular cells. Within each cell $C_\alpha$, the
prediction function is typically constant: $f_\alpha(x) = c_\alpha$. For
regression, $c_\alpha$ is the average of response values in the cell; for
classification, it is often the modal class or a probability vector over
classes.

Multivariate Adaptive Regression Splines (MARS) \cite{friedman1991multivariate}
construct axis-aligned hyperrectangular partitions through products of hinge
functions. A MARS model takes the form
$f(x) = \beta_0 + \sum_{m=1}^M \beta_m B_m(x)$, where each basis function $B_m$
is a product of terms $(x^{(j)} - t)_+ = \max(0, x^{(j)} - t)$ or
$(t - x^{(j)})_+$. This construction implicitly partitions the domain into
cells where $f$ has a fixed polynomial form. Unlike tree methods, MARS
enforces continuity across cell boundaries: for adjacent cells $C_\alpha$ and
$C_\beta$, we have $\lim_{x \to x'} f(x) = \lim_{x \to x'} f(x)$ for
$x' \in \partial C_\alpha \cap \partial C_\beta$. This combination of
piecewise polynomial structure and continuity constraints distinguishes MARS
as a smooth approximation method despite its partition-based foundation.

Oblique decision trees \cite{murthy1994system} generalize the partition
structure beyond axis-aligned boundaries by allowing splits based on linear
combinations of features. An internal node performs a test of the form
$\beta^T x \leq t$, where $\beta \in \mathbb{R}^n$ is a weight vector not
restricted to coordinate directions. This leads to partitions by arbitrarily
oriented hyperplanes, resulting in general convex polyhedral cells rather than
hyperrectangles. The prediction functions within each cell are typically
constant, though more expressive variants with piecewise linear models have been
developed \cite{bennett1996linear}.

Voronoi tessellations provide another approach to creating convex polyhedral
partitions, based on proximity to prototype points rather than sequential
splitting. Given prototypes $\mu_1, \ldots, \mu_K \in \mathcal{D}$, the Voronoi
cell $C_k$ is defined as:
$$
C_k = \left\{ x \in \mathcal{D} : \|x - \mu_k\| \leq \|x - \mu_j\| \text{ for all } j \neq k \right\}.
$$
Under the Euclidean norm, this yields convex polygonal or polyhedral cells whose
boundaries are perpendicular bisectors between prototype pairs. Local prediction
models within each cell range from constants (as in nearest neighbor methods
\cite{cover1967nearest}) to linear or nonlinear functions (as in local
regression \cite{cleveland1988locally}). Vector quantization
\cite{gray1998quantization} and learning vector quantization
\cite{kohonen1995learning} adaptively update the prototypes to optimize the
partition for prediction accuracy.

Support vector machines with local models \cite{bottou1992local} implement
partitioning through a two-stage process: first subdividing the feature space
using clustering or other heuristics, then fitting an SVM within each region.
These partitions can employ hard boundaries with distinct regions or soft
transitions with overlapping influence functions, allowing adaptation to
heterogeneous data distributions where global models fail to capture local
patterns.

Recursive partitioning methods \cite{zhang1998recursive} provide a unifying
framework that encompasses many of the approaches above. These methods divide
the feature space recursively using flexible splitting criteria: single
variables (axis-aligned trees), linear combinations (oblique trees), or more
complex functions. The partition refinement continues until a stopping criterion
is met, such as minimum cell occupancy or lack of improvement in prediction
accuracy. This framework highlights how different choices of splitting functions
and local models lead to the variety of partition-based methods in practice.

While recursive partitioning methods typically proceed top-down from a single
cell containing all data, bottom-up approaches offer an alternative construction
principle. These methods begin with a fine-grained partition and iteratively
merge adjacent cells based on optimality criteria. For instance, one could
initialize with a Voronoi tessellation using all data points as prototypes, then
repeatedly merge adjacent cells when the merged predictor achieves lower error
than the individual cells. This approach generalizes hierarchical clustering by
replacing Ward's variance-minimization criterion with a prediction-based merging
criterion that directly optimizes model performance. Similar principles appear
in region growing algorithms from image segmentation \cite{adams1994seeded} and
agglomerative information bottleneck methods \cite{slonim1999agglomerative},
though adapted here for supervised learning with explicit local prediction
functions. Such agglomerative methods naturally incorporate local geometry
through the adjacency structure and can discover partitions that balance model
complexity with predictive accuracy. The choice between top-down splitting and
bottom-up merging reflects different inductive biases: splitting emphasizes
global structure and interpretability, while merging captures local similarities
and can adapt to varying data density.

The partition models described above exhibit a progression in geometric and
functional complexity. Axis-aligned methods (trees, MARS) create
hyperrectangular cells that align with coordinate axes, facilitating
interpretation but potentially requiring many cells for diagonal boundaries.
General linear methods (oblique trees, Voronoi tessellations) allow arbitrary
convex polyhedral cells, better adapting to data geometry at the cost of
interpretability. Local model complexity ranges from constants through
polynomials to general nonlinear functions, while continuity properties vary
from the discontinuous predictions of trees to the smooth transitions of MARS
and kernel-based methods. These structural differences directly influence the
structure of the induced simplicial complexes, as explored in subsequent
sections.

\subsubsection*{2.1.2 Regularization in Partition Models}

Why do partition models with excellent training accuracy often fail
catastrophically on new data? A decision tree might create thousands of tiny
cells to perfectly classify training points, yet misclassify simple test cases.
Current regularization approaches limit tree depth or penalize the number of
leaves, but these constraints address symptoms rather than causes. The real
problem lies deeper: pathological geometric structures that we cannot detect or
penalize with existing methods.

Consider a partition where one cell occupies 99\% of the feature space while
hundreds of tiny cells fragment the remaining 1\%. Traditional regularization
counts each cell equally, missing this extreme imbalance. Or imagine cells
meeting at sharp, knife-like angles that create unstable decision boundaries.
Standard penalties on model complexity fail to distinguish this problematic
geometry from a well-balanced partition with the same number of cells. These
limitations suggest we need regularization approaches that directly target
geometric properties.

Regularization in partition models serves to control model complexity and
prevent overfitting, yet the form of the penalty $R_\alpha$ varies substantially
across model classes. In tree-based methods, regularization typically constrains
tree complexity through structural constraints rather than explicit penalty
terms. Cost-complexity pruning adds a penalty term $\alpha |T|$ to the empirical
loss, where $|T|$ denotes the number of terminal nodes. This can be interpreted
as assigning a uniform regularization cost $R_\alpha = \alpha$ to each leaf
cell, discouraging excessively deep trees \cite{breiman1984classification}.

For MARS models, regularization penalizes both the number of basis functions and
their interaction complexity. The penalty typically takes the form
$$
R_{\text{MARS}} = \lambda_1 M + \lambda_2 \sum_{m=1}^M d_m,
$$
where $M$ is the number of basis functions and $d_m$ denotes the number of
distinct variables in the $m$-th basis function, reflecting its interaction
order \cite{friedman1991multivariate,hastie2009elements}. This formulation
discourages both model size and high-order interactions, promoting
interpretability while controlling variance.

Similar limitations appear across other partition methods. Voronoi-based
approaches typically regularize only the number or positions of prototype
points, without considering the geometry of the resulting cells
\cite{kohonen1995learning}. Oblique trees typically focus on finding optimal
linear combinations for splits without explicit geometric regularization of the
resulting partition structure \cite{murthy1994system,heath1993induction}.
Support vector machines with local models often inherit regularization from
their constituent SVMs, focusing on margin maximization within cells rather than
the quality of the partition structure itself \cite{bottou1992local}.

Classical regularization approaches thus target either the number of cells or
the smoothness of functions within them, but rarely leverage the geometric
structure of the partition. This limitation becomes apparent when considering
pathological partitions: a model might create cells with extreme volume
disparities, indicating potential overfitting to sparse data regions, or
generate boundaries meeting at sharp angles, suggesting non-smooth decision
surfaces that may generalize poorly. A partition with large
surface-area-to-volume ratios might indicate excessive complexity, fragmenting
the feature space unnecessarily.

These observations motivate the development of geometric regularization
approaches that directly leverage the spatial properties of partitions,
including natural geometric quantities such as cell volumes, boundary areas, and
angles between adjacent cells.

\subsection*{2.2 Geometry of Partition Models}

When two cells in a partition meet along a boundary, what geometric information
does this relationship contain? Consider two adjacent regions created by a
decision tree. We know they share a boundary, but we might also ask: How large
is this shared boundary? Does a long boundary indicate a gradual transition
while a small one suggests an abrupt change? If we compare the volumes of these
cells, does a large disparity signal that one region is overly specialized?
These questions reveal that partitions contain geometric information beyond mere
adjacency.

The situation becomes richer when three or more cells meet at a point. The
angles at which they meet encode information about the local configuration that
pairwise relationships cannot capture. These geometric properties, such as
boundary areas, cell volumes, and meeting angles, provide insights into model
behavior that current regularization approaches overlook.

This section develops a mathematical framework to capture and exploit these
geometric structures. We begin by showing how partitions naturally give rise to
simplicial complexes that encode adjacency relationships (Section 2.2.1). We
then enrich these structures with geometric information through Riemannian
metrics (Section 2.2.2). This framework extends naturally to ensemble methods
and reveals gradient boosting as a geometric flow (Sections 2.2.3-2.2.4).
Finally, we demonstrate how these geometric insights enable new regularization
strategies (Section 2.2.5).

\subsubsection*{2.2.1 From Partitions to Simplicial Complexes}

How can we mathematically capture the relationships between cells in a partition
beyond simple adjacency? When three or more cells meet at a point, they form a
configuration that pairwise adjacency cannot describe. In three dimensions, four
cells might meet at a vertex, creating a tetrahedral arrangement. The angles and
orientations of this meeting contain geometric information essential for
understanding the partition's structure, yet standard graph representations
discard these higher-order relationships.

This limitation affects the analysis of partition models and may impact the
predictive performance of the partition model. A graph tells us which cells are
neighbors but not how they meet in space. Are the cells arranged symmetrically
around their common vertex, or do they form sharp, unstable configurations? Do
many cells crowd together in one region while others stand isolated? These
geometric patterns may impact model behavior and generalization, but we need
richer mathematical structures to detect and quantify them.

To understand how geometric relationships in partitions can be captured
mathematically, we begin with a familiar example: the relationship between
Voronoi diagrams and Delaunay triangulations. Given points
$X = \{x_1, x_2, \ldots, x_k\}$ in $\mathcal{D}$, the Voronoi tessellation
\cite{voronoi1908nouvelles,aurenhammer1991voronoi} partitions $\mathcal{D}$ into
cells, where each cell $C_i$ consists of all points closer to $x_i$ than to any
other point in $X$. This partition has a natural adjacency structure: two cells
$C_i$ and $C_j$ are adjacent when they share a common boundary of dimension
$n-1$. The Delaunay triangulation \cite{delaunay1934sphere} captures this
adjacency by connecting points $x_i$ and $x_j$ with an edge whenever their
corresponding Voronoi cells are adjacent.

This construction extends beyond edges. When three Voronoi cells in
$\mathbb{R}^2$ meet at a common vertex, their generating points form a triangle
in the Delaunay triangulation. More generally, whenever $d+1$ Voronoi cells
share a common face of dimension $n-d$, their generating points form a
$d$-simplex in the Delaunay triangulation
\cite{okabe2000spatial,preparata1985computational}. The result is a simplicial
complex that encodes the full combinatorial structure of the Voronoi partition.

The Delaunay triangulation exhibits a crucial property: if we have a triangle
with vertices $\{x_i, x_j, x_k\}$, then we also have all three edges
$\{x_i, x_j\}$, $\{x_i, x_k\}$, and $\{x_j, x_k\}$, as well as all three
vertices. Similarly, a tetrahedron includes all its triangular faces, edges, and
vertices. This closure property characterizes the structure of a simplicial
complex. Formally, a simplicial complex on a vertex set $V$ is a collection
$K = \bigcup_{p=0}^d K_p$, where each $K_p$ consists of $(p+1)$-element subsets
of $V$ called $p$-simplices \cite{munkres1984elements,hatcher2002algebraic}. The
defining property is closure under taking subsets: if $\sigma \in K$ and
$\tau \subseteq \sigma$, then $\tau \in K$. This ensures that the boundary of
every simplex is contained in the complex, guaranteeing that simplicial
complexes have no self-intersections and maintain a coherent combinatorial
structure.

The key insight is that this process applies to any partition, not just Voronoi
tessellations. Given any partition $\mathcal{P} = \{C_\alpha\}_{\alpha \in I}$
of a bounded domain $\mathcal{D} \subseteq \mathbb{R}^n$, we associate each cell
$C_\alpha$ with a vertex $v_\alpha$, yielding vertex set
$V = \{v_\alpha : \alpha \in I\}$. We construct the nerve complex
$\mathcal{N}_{\bullet}(\mathcal{P})$ to capture multi-way adjacency relationships:
$$
\mathcal{N}_p(\mathcal{P}) = \{S \subset \mathcal{P} : |S| = p+1 \text{ and } \bigcap_{C \in S} \overline{C} \ne \emptyset\}
$$
where $\overline{C}$ denotes the closure of cell $C$. The sequence
$$
\mathcal{N}_{\bullet}(\mathcal{P}) = \left(\mathcal{N}_0(\mathcal{P}), \mathcal{N}_1(\mathcal{P}), \mathcal{N}_2(\mathcal{P}), \ldots \right)
$$
forms a simplicial complex. This simplicial complex constructed from the partition is called the
nerve complex, and the construction that produces it is known as the nerve
construction \cite{borsuk1948theorem,leray1945forme,weil1952varietes}.

This construction generalizes beyond partitions. A covering of a space $X$ is a
collection of sets $\mathcal{U} = \{U_\alpha\}_{\alpha \in A}$ such that
$X = \bigcup_{\alpha \in A} U_\alpha$. Unlike partitions, coverings allow
arbitrarily complex overlaps between sets. Given any covering $\mathcal{U}$, we can form
the nerve complex $\mathcal{N}_{\bullet}(\mathcal{U})$ by including a
$k$-simplex whenever $k+1$ sets from the covering have nonempty intersection.
The partition case is simply a special instance where the sets have disjoint
interiors.

In the case of a Voronoi tessellation with cells in general position, the nerve complex
coincides precisely with the Delaunay triangulation
\cite{delaunay1934sphere,aurenhammer1991voronoi}, providing a geometric
realization of this abstract combinatorial structure. The nerve complex thus
serves as a geometric descriptor that reveals structural properties of the
partition amenable to analysis and regularization. The partition itself, with its
potentially complex local prediction functions, remains the primary object.

The construction yields different structures for different partition types
encountered in machine learning. For axis-aligned decision trees in
$\mathbb{R}^n$, the hyperrectangular cells meet along $(n-1)$-dimensional
facets, $(n-2)$-dimensional edges, and so forth, down to vertices where $2^n$
cells meet. Each such vertex contributes a $(2^n-1)$-simplex to the nerve
complex, making these complexes $(2^n-1)$-dimensional. MARS models
\cite{friedman1991multivariate} create similar hyperrectangular partitions but
with a crucial difference that the cell maps $f_\alpha$ are continuous along
cell boundaries (not affecting the geometric structure of the nerve complex). On
the other hand, the dimension of the Delaunay complex associated with a Voronoi
tessellation of an $n$-dimensional region of $\mathbb{R}^{n}$ is always $n$.
Oblique decision trees \cite{murthy1994system} and other methods using general
hyperplane splits create convex polyhedral cells that can meet in various
configurations. While axis-aligned splits necessarily create vertices where
$2^n$ cells meet, oblique splits allow more flexibility. In general position,
$n$ hyperplanes through a point still create $2^n$ regions, but the specific
configuration of splits in a trained model may result in fewer cells meeting at
any given point, potentially yielding nerve complexes of lower dimension. The
dimension of the resulting simplicial complex reflects both the ambient
dimension and the specific geometric configuration of the partition.

\subsubsection*{2.2.2 Riemannian Structure on Partition Simplicial Complexes}

When three cells meet at a point, what distinguishes a balanced configuration
from a pathological one? The simplicial complex captures which cells are
adjacent, but this combinatorial information alone cannot distinguish between
fundamentally different geometric configurations. Consider three cells meeting
at a point. In one configuration, they meet symmetrically with 120-degree angles
between their boundaries, suggesting a balanced partition. In another, two cells
form a narrow wedge against the third, creating sharp angles that may indicate
overfitting or poor generalization. A graph representation treating these
configurations identically would miss this crucial geometric distinction.

How can we capture these angular relationships? What about the relative sizes of
shared boundaries? When cells meet at unusual angles or have vastly different
volumes, these geometric features may signal problems with the partition
structure. To detect and quantify such properties, we need to enrich our
simplicial complex with geometric information.

Just as Riemannian metrics enrich smooth manifolds with notions of length,
angle, and curvature, partition complexes can be equipped with a Riemannian
structure. This additional structure assigns inner products to the spaces
associated with each simplex, enabling us to measure angles between edges,
volumes of cells, and areas of boundaries. These measurements transform the
discrete complex into a computational framework where geometric properties
become accessible for analysis and regularization.

\paragraph{From Smooth to Discrete Riemannian Geometry}

To understand how geometric concepts translate from smooth manifolds to
simplicial complexes, we begin with familiar ideas from linear algebra. In
Euclidean space $\mathbb{R}^n$, the standard inner product
$\langle v, w \rangle = v^T w$ enables us to compute lengths
$\|v\| = \sqrt{\langle v, v \rangle}$ and angles
$\cos(\theta) = \langle v, w \rangle / (\|v\| \|w\|)$. Any positive definite
matrix $A$ defines an alternative inner product
$\langle v, w \rangle_A = v^T A w$, inducing a different geometry where lengths
and angles are measured differently.

A smooth manifold like the unit sphere $S^2$ in $\mathbb{R}^3$ inherits
geometric structure from the ambient space. At each point $x \in S^2$, there
exists a tangent space $T_x S^2$ consisting of all vectors tangent to the sphere
at $x$. The Euclidean inner product of $\mathbb{R}^3$ restricts to an inner
product on each tangent space. While each individual tangent space just has the
standard inner product, the crucial point is that these tangent spaces
themselves are rotating as we move around the sphere. This variation of the
tangent spaces from point to point is what creates curvature.

To see this concretely: if you parallel transport a vector around a closed loop
on $S^2$, it returns rotated relative to its starting position. This rotation is
a manifestation of curvature, even though at each point we're using the
"standard" inner product on that tangent space. The Riemannian metric encodes
not just the inner products, but implicitly the entire geometric relationship
between nearby tangent spaces.

The key insight is that simplicial complexes possess an analogous, though
discrete, structure. While a point on a smooth manifold has infinitesimally
small neighborhoods, a vertex $v$ in a simplicial complex $K$ has a discrete
neighborhood: its star $\text{Star}(v) = \{\sigma \in K : v \in \sigma\}$,
consisting of all simplices containing $v$. The edges emanating from $v$ play
the role of "discrete tangent vectors," and the space of formal linear
combinations of these edges, denoted $\Lambda_1(\text{Star}(v))$, serves as the
discrete analogue of a tangent space.

This correspondence extends systematically. The exterior powers
$\Lambda^p T_x M$ of tangent spaces, which encode $p$-dimensional infinitesimal
volumes, correspond to spaces $\Lambda_p(\text{Star}(v))$ of $p$-chains in the
star. Similarly, differential forms on manifolds, which are integrated over
submanifolds, correspond to cochains on simplicial complexes, which are
evaluated on chains.

\paragraph{Building the Discrete Riemannian Structure} \label{sec:discrete-riemannian}

We now formalize these ideas. For a simplicial complex $K$, the space of
oriented $p$-chains supported on a subcomplex $L \subseteq K$ is:
$$
\Lambda_p(L) = \left\{\sum_{\sigma \in L_p} a_\sigma \cdot v_{i_0} \wedge v_{i_1} \wedge \cdots \wedge v_{i_p} : a_\sigma \in \mathbb{R}\right\} / \sim
$$
where $L_p$ denotes the set of $p$-simplices in $L$, each simplex
$\sigma = \{v_{i_0}, v_{i_1}, \ldots, v_{i_p}\}$ is represented by the formal
wedge product $v_{i_0} \wedge v_{i_1} \wedge \cdots \wedge v_{i_p}$, and $\sim$
is the equivalence relation
$$
v_{\pi(0)} \wedge v_{\pi(1)} \wedge \cdots \wedge v_{\pi(p)} \sim \text{sgn}(\pi) \cdot v_0 \wedge v_1 \wedge \cdots \wedge v_p
$$
for any permutation $\pi$. This antisymmetry relation ensures that the wedge
product encodes orientation: even permutations preserve the orientation while
odd permutations reverse it. Our notation $\Lambda_p(L)$, rather than the
classical $C_p(L)$, emphasizes that simplices are treated as antisymmetric
tensors, reflecting the deep connection between oriented simplicial chains and
exterior algebra.

Having established the chain spaces $\Lambda_p(L)$ for subcomplexes of $K$, we
can now formalize the Riemannian structure that encodes geometric information
throughout the complex. A Riemannian simplicial complex consists of a simplicial complex $K$ equipped
with a compatible family of inner products
$$
\langle \cdot, \cdot \rangle_{\sigma,p} : \Lambda_p(\overline{\text{Star}}(\sigma)) \times \Lambda_p(\overline{\text{Star}}(\sigma)) \rightarrow \mathbb{R}
$$
defined for all simplices $\sigma \in K$ and all dimensions $p \geq 0$, where
$\overline{\text{Star}}(\sigma)$ denotes the closed star of $\sigma$ (the
smallest subcomplex containing all simplices that have $\sigma$ as a face).

The compatibility of these inner products is ensured through two fundamental
conditions. First, the consistency condition requires that when a $p$-chain
$\tau$ belongs to multiple star complexes, its self-inner product remains
invariant:
$$
\langle \tau, \tau \rangle_{\sigma,p} = \langle \tau, \tau \rangle_{\sigma',p}
$$
for all
$\tau \in \overline{\text{Star}}(\sigma) \cap \overline{\text{Star}}(\sigma')$.
This ensures that geometric quantities such as lengths and volumes are
well-defined regardless of the local context in which they are computed.

Second, the restriction compatibility addresses the hierarchical nature of
simplicial complexes. When $\tau \subseteq \sigma$, the containment
$\overline{\text{Star}}(\sigma) \subseteq \overline{\text{Star}}(\tau)$ induces a natural restriction
map
$$
r_{\sigma,\tau}: \Lambda_p(\overline{\text{Star}}(\tau)) \to \Lambda_p(\overline{\text{Star}}(\sigma))
$$
For chains $\rho, \rho' \in \Lambda_p(\overline{\text{Star}}(\sigma))$, the compatibility
requires
$$
\langle \rho, \rho' \rangle_{\sigma,p} = \langle \rho, \rho' \rangle_{\tau,p}.
$$
This condition ensures that geometric measurements respect the inclusion
relationships between simplices, allowing local geometric information to
assemble consistently into a global structure. The compatibility conditions
reveal that for each $p \geq 0$, the collection of chain spaces
$\{\Lambda_p(\overline{\text{Star}}(\sigma))\}_{\sigma \in K}$ forms a sheaf
over $K$, with restriction maps satisfying the sheaf axioms
\cite{bredon1997sheaf,kashiwara2003sheaves}.

These compatibility conditions provide more than technical coherence; they
encode a fundamental principle of geometric consistency. Just as a Riemannian
manifold requires that tangent space inner products vary smoothly, a Riemannian
simplicial complex requires that star complex inner products fit together
compatibly. This compatibility enables the passage from local geometric data to
global geometric invariants, a theme that pervades our subsequent constructions.

The relationship to weighted graphs illuminates the generality of this
framework. A one-dimensional simplicial complex, consisting only of vertices and
edges, equipped with a Riemannian structure, extends the notion of a weighted
graph. At each vertex $v$, the inner product on
$\Lambda_0(\overline{\text{Star}}(v)) = \text{span}\{v\}$ assigns a vertex
weight. The inner product on $\Lambda_1(\overline{\text{Star}}(v))$ provides
richer information than simple edge weights: it defines a Gram matrix whose
diagonal entries determine edge weights and whose off-diagonal entries encode
angular relationships between edges meeting at $v$. For an edge
$e = v \wedge w$, the consistency condition ensures
$\langle e, e \rangle_{v,1} = \langle e, e \rangle_{w,1}$, yielding a unique
edge weight $w_e = \langle e, e \rangle_{v,1}^{1/2}$. However, for two edges
$e_1 = v \wedge w_1$ and $e_2 = v \wedge w_2$ meeting at $v$, the inner product
$\langle e_1, e_2 \rangle_{v,1}$ captures their geometric relationship at $v$,
information absent in classical weighted graphs. This additional structure
becomes essential when modeling systems where angular relationships matter, such
as molecular conformations or mesh geometries.

The passage from smooth to discrete is justified by deep results in differential
topology. Every smooth manifold admits triangulations that approximate its
geometry arbitrarily well \cite{whitehead1940c1,munkres1984elements}, and the
discrete operators converge to their smooth counterparts under refinement
\cite{hirani2003discrete,desbrun2005discrete}. This convergence theory validates
our use of simplicial complexes for analyzing partition models while providing
concrete computational tools.

Central to this correspondence is the preservation of fundamental calculus
structures in the discrete setting. Just as smooth manifolds support
differential forms and the exterior derivative, simplicial complexes support
cochains and coboundary operators \cite{bott1982differential,
  warner1983foundations}. The boundary operator
$$
\partial_p : \Lambda_p(K) \to \Lambda_{p-1}(K)
$$
generalizes the notion of taking boundaries:
$$
\partial_p(v_0 \wedge \cdots \wedge v_p) = \sum_{j=0}^p (-1)^j v_0 \wedge \cdots \wedge \hat{v}_j \wedge \cdots \wedge v_p
$$
where $\hat{v}_j$ indicates omission of vertex $v_j$. Its dual, the coboundary
operator
$$
\delta_p : \Omega^p(K) \to \Omega^{p+1}(K)
$$
on cochains (discrete differential forms), is defined by
$$
(\delta_p \omega)(\sigma) = \omega(\partial_{p+1} \sigma).
$$
This definition ensures that the discrete Stokes' theorem holds:
$$
\langle \delta_p \omega, \sigma \rangle = \langle \omega, \partial_{p+1} \sigma \rangle
$$
where $\langle \cdot, \cdot \rangle$ denotes the pairing between cochains and chains.

The profound connection to smooth differential geometry emerges through
integration. Under the de Rham map that sends a smooth differential form
$\alpha$ to the cochain $\sigma \mapsto \int_\sigma \alpha$, the smooth exterior
derivative $d$ corresponds to the discrete coboundary $\delta$. The defining
property of $\delta$ then becomes a consequence of the classical Stokes'
theorem:
$$
\int_\sigma d\alpha = \int_{\partial \sigma} \alpha
$$
This correspondence ensures our discrete framework maintains the essential
structure of calculus, with the algebraic relation $\delta \circ \delta = 0$
reflecting the geometric fact that $d \circ d = 0$.

This preservation of Stokes' theorem is not merely a formal nicety, it ensures
that our discrete geometric constructions respect conservation laws and
variational principles that govern physical systems. When we construct discrete
harmonic extensions and smoothing operators on simplicial complexes in Section
2.3, this foundational correspondence guarantees that our methods inherit the
stability and convergence properties of their smooth counterparts. Moreover,
this framework provides the foundation for richer geometric analyses, including
discrete gradient flows and Morse-Smale complexes, that reveal additional
structure in machine learning models.

\paragraph{Geometric Structure of Partition Complexes}

With this formal framework established, we can now specify the Riemannian
structure for partition-induced simplicial complexes. For partition-induced
complexes, the Riemannian structure encodes the geometry of cell boundaries and
their configurations. We begin by defining inner products on 0-chains
(vertices). For a vertex $v_i$ corresponding to cell $C_i$:
$$
\langle v_i, v_i \rangle_{v_i,0} = \text{vol}_n(C_i)
$$
This choice reflects that larger cells carry more "weight" in the geometric structure,
analogous to mass distributions in physics.

For 1-chains, the fundamental building blocks are inner products between edges
meeting at vertices. For edges $e = v_i \wedge v_j$ and $e' = v_i \wedge v_k$ in
$\overline{\text{Star}}(v_i)$, representing cells $C_i, C_j, C_k$ with $C_i$
adjacent to both $C_j$ and $C_k$, we consider the boundaries
$F_{ij} = \overline{C_i} \cap \overline{C_j}$ and
$F_{ik} = \overline{C_i} \cap \overline{C_k}$.

When these boundaries are $(n-1)$-dimensional (proper facets), we define:
\begin{equation} \label{edge.vi.inner.prod:eq}
  \langle e, e' \rangle_{v_i,1} = \sqrt{\text{vol}_{n-1}(F_{ij}) \cdot \text{vol}_{n-1}(F_{ik})} \cdot \cos\theta_{jk}
\end{equation}
where $\theta_{jk}$ is the interior dihedral angle of cell $C_i$ between the
hyperplanes $H_{ij}$ and $H_{ik}$ containing facets $F_{ij}$ and $F_{ik}$.

To compute this angle in practice, we use the formula:
$$
\cos\theta_{jk} = \frac{\mathbf{n}_{ij} \cdot \mathbf{n}_{ik}}{|\mathbf{n}_{ij}||\mathbf{n}_{ik}|}
$$
where $\mathbf{n}_{ij}$ is the normal vector to $H_{ij}$ pointing into $C_i$ and
$\mathbf{n}_{ik}$ is the normal vector to $H_{ik}$ pointing away from $C_i$.

This convention arises from the following geometric observation: The interior
angle is the length of the arc on the unit circle (in the plane perpendicular to
both hyperplanes) that lies within cell $C_i$. When we parallel transport the
inward-pointing normal $\mathbf{n}_{ij}$ along this interior arc, we arrive at
the outward-pointing normal $\mathbf{n}_{ik}$. Therefore, the angle between
these specifically oriented normals equals the interior dihedral angle.

When the boundaries have dimension less than $n-1$, we adopt the convention that
the inner product is zero, reflecting that the cells do not share proper facets
at vertex $v_i$.

The formula (\ref{edge.vi.inner.prod:eq}) captures both metric information
(facet areas) and angular relationships (dihedral angles). The angle term
$\cos\theta_{jk}$ quantifies the relative orientation of adjacent cells:
perpendicular boundaries yield zero inner product, while aligned boundaries
produce large values.

For the important case of hyperrectangular partitions arising from tree-based
models and MARS, all boundaries are aligned with coordinate axes. Consequently,
any two distinct boundaries meet at right angles, giving $\cos\theta_{jk} = 0$
for all $j \neq k$. This means the Gram matrices for such partitions are
diagonal, with diagonal entries determined solely by boundary volumes. This
geometric signature, zero off-diagonal entries, is characteristic of
axis-aligned partition methods and simplifies many computations.

We use the geometric mean
$\sqrt{\text{vol}_{n-1}(F_{ij}) \cdot \text{vol}_{n-1}(F_{ik})}$ to combine
boundary volumes, which maintains homogeneity under scaling. Alternative
choices, such as arithmetic or harmonic means, would lead to different geometric
weightings and potentially different regularization properties.

This construction generalizes naturally to higher dimensions through Gram
determinants. For a simplex $\sigma \in K$ and oriented edges
$e = \sigma \wedge v$ and $e' = \sigma \wedge v'$ in
$\Lambda_{\dim(\sigma)+1}(\overline{\text{Star}}(\sigma))$, we define:
$$\langle \sigma \wedge v, \sigma \wedge v' \rangle_{\sigma,\dim(\sigma)+1} = \text{vol}_{n-\dim(\sigma)-1}(F_{e})^{1/2} \cdot \text{vol}_{n-\dim(\sigma)-1}(F_{e'})^{1/2} \cdot \cos\phi,$$
where
$$F_e = \bigcap\{C_\alpha : v_\alpha \in \{\text{vertices of } \sigma\} \cup \{v\}\}$$
and similarly for $F_{e'}$. The volume term $\text{vol}_{n-\dim(\sigma)-1}(F)$
denotes the $(n-\dim(\sigma)-1)$-dimensional Lebesgue measure of the face when
this dimension is positive.

To understand the angle $\phi$, note that
$$F_\sigma = \bigcap\{C_\alpha : v_\alpha \in \{\text{vertices of } \sigma\}\}$$
contains both $F_e = F_\sigma \cap C_v$ and $F_{e'} = F_\sigma \cap C_{v'}$,
where $C_v$ and $C_{v'}$ are the cells corresponding to vertices $v$ and $v'$
respectively.

Within the affine span $H_\sigma$ of $F_\sigma$ (which has dimension
$n - \dim(\sigma)$), $H_e = \text{span}(F_e)$ and $H_{e'} = \text{span}(F_{e'})$
are hyperplanes of codimension 1 in $H_\sigma$. The interior dihedral angle
$\phi$ between these hyperplanes is measured as the arc length on the unit
circle in the plane perpendicular to both $H_e$ and $H_{e'}$ within $H_\sigma$.

Following our convention for interior angles:
$$\cos\phi = \frac{\mathbf{n}_e \cdot \mathbf{n}_{e'}}{|\mathbf{n}_e||\mathbf{n}_{e'}|}$$
where $\mathbf{n}_e$ and $\mathbf{n}_{e'}$ are normal vectors to $H_e$ and
$H_{e'}$ within $H_\sigma$, with $\mathbf{n}_e$ pointing toward $C_v$ and
$\mathbf{n}_{e'}$ pointing away from $C_{v'}$.

This construction naturally handles all possible dimensional configurations.
When $\dim(\sigma) = 0$ (vertices), we recover the $(n-1)$-dimensional facet
intersections from our earlier edge construction. As $\dim(\sigma)$ increases,
the dimension $n-\dim(\sigma)-1$ of the intersection decreases accordingly,
capturing the natural constraint that when more cells meet simultaneously, they
do so in lower-dimensional manifolds. For instance, in $\mathbb{R}^3$, two cells
meet along 2D faces, three cells meet along 1D edges, and four cells meet at 0D
points. The formula automatically adapts to measure volumes in the appropriate
dimension for each configuration.

In the special case where $n-\dim(\sigma)-1 = 0$, indicating point
intersections, we adopt the convention that $\text{vol}_0(F) = 1$ if $F$ is
non-empty and $\text{vol}_0(F) = 0$ otherwise. For such zero-dimensional
intersections, we set $\cos\phi = 1$ when both intersections are non-empty,
reflecting the discrete nature of point intersections where angular measurements
are not meaningful. This convention maintains consistency across dimensions
while acknowledging the fundamental difference between continuous geometric
measurements and discrete intersection patterns.

For $p$-simplices $\rho = \sigma \wedge v_1 \wedge \cdots \wedge v_q$ and
$\rho' = \sigma \wedge v'_1 \wedge \cdots \wedge v'_q$ in
$\Lambda_p(\overline{\text{Star}}(\sigma)))$, where $q = p - \dim(\sigma)$, the inner
product is defined as:
$$\langle \rho, \rho' \rangle_{\sigma,p} = \det(G_{\rho,\rho'})$$
where the Gram matrix has entries
$$
[G_{\rho,\rho'}]_{ij} = \langle \sigma \wedge v_i, \sigma \wedge v'_j \rangle_{\sigma,\dim(\sigma)+1}.
$$
This determinant formula, standard in multilinear algebra, extends the notion of
volume from Euclidean space to our discrete setting. This construction
represents the unique extension of inner products from edges to
higher-dimensional simplices that respects the multilinear structure of exterior
powers. The determinant formula ensures that the inner product on pp p-chains is
induced from the inner products on edges through the natural algebraic
structure, preserving orientation and capturing pp p-dimensional volumes
\cite{federer1996geometric, whitney1957geometric}.

\paragraph{Example: Three Cells Meeting at a Vertex}
We can now make precise the geometric distinctions raised at the beginning of
this section. Consider three cells $C_1$, $C_2$, $C_3$ meeting at a vertex $v$
in $\mathbb{R}^2$. The three edges at $v$ are: $e_1 = v \wedge v_1$,
$e_2 = v \wedge v_2$, and $e_3 = v \wedge v_3$. For the 2-simplex
$\rho = v \wedge v_1 \wedge v_2$, the Gram matrix is:
$$G_\rho = \begin{pmatrix}
\langle e_1, e_1 \rangle_v & \langle e_1, e_2 \rangle_v \\
\langle e_2, e_1 \rangle_v & \langle e_2, e_2 \rangle_v
\end{pmatrix} = \begin{pmatrix}
\text{vol}_{n-1}(F_{v,v_1}) & \gamma_{12} \\
\gamma_{12} & \text{vol}_{n-1}(F_{v,v_2})
\end{pmatrix}$$
where $\gamma_{12} = \sqrt{\text{vol}_{n-1}(F_{v,v_1}) \cdot \text{vol}_{n-1}(F_{v,v_2})} \cos\theta_{12}$.
The determinant captures both the areas of boundaries and their angular
relationship. For axis-aligned partitions where $\theta_{12} = 90°$, the
off-diagonal terms vanish, yielding
$$
\det(G_\rho) = \text{vol}_{n-1}(F_{v,v_1}) \cdot \text{vol}_{n-1}(F_{v,v_2}),
$$
a simple product of boundary lengths.

\paragraph{Example: Riemannian Structure on Tree Partitions}
Consider a decision tree in $\mathbb{R}^2$ creating four cells through
perpendicular splits at $x_1 = 2$ and $x_2 = 1$ (for $x_1 < 2$) and
$x_2 = 3$ (for $x_1 \geq 2$).

For the 0-chain structure, if the domain is $[0,4] \times [0,4]$, then:
\begin{itemize}[label=-, topsep=-5pt, partopsep=0pt, parsep=0pt, itemsep=0pt, after=\vspace{0.2cm}]
  \item $\langle v_1, v_1 \rangle_{v_1,0} = \text{vol}_2(C_1) = 2 \times 1 = 2$
  \item $\langle v_2, v_2 \rangle_{v_2,0} = \text{vol}_2(C_2) = 2 \times 3 = 6$
  \item $\langle v_3, v_3 \rangle_{v_3,0} = \text{vol}_2(C_3) = 2 \times 1 = 2$
  \item $\langle v_4, v_4 \rangle_{v_4,0} = \text{vol}_2(C_4) = 2 \times 3 = 6$
\end{itemize}

For the 1-chain structure at vertex $v_1$ representing cell
$C_1 = \{x_1 < 2, x_2 < 1\}$:
\begin{itemize}[label=-, topsep=-5pt, partopsep=0pt, parsep=0pt, itemsep=0pt, after=\vspace{0.2cm}]
  \item The boundary $F_{12}$ with cell $C_2 = \{x_1 < 2, x_2 \geq 1\}$ is the line segment $\{(x_1, 1) : x_1 < 2\}$ with length 2.
  \item The boundary $F_{13}$ with cell $C_3 = \{x_1 \geq 2, x_2 < 1\}$ is the line segment $\{(2, x_2) : x_2 < 1\}$ with length 1.
\end{itemize}
These boundaries meet at right angles at point $(2,1)$, so $\cos\theta_{23} = 0$. The Gram matrix at $v_1$ becomes:
$$G_{v_1} = \begin{pmatrix}
\langle e_{12}, e_{12} \rangle_{v_1,1} & \langle e_{12}, e_{13} \rangle_{v_1,1} \\
\langle e_{13}, e_{12} \rangle_{v_1,1} & \langle e_{13}, e_{13} \rangle_{v_1,1}
\end{pmatrix} = \begin{pmatrix} 2 & 0 \\ 0 & 1 \end{pmatrix}$$

The condition number $\kappa(G_{v_1}) = 2$ indicates mild anisotropy in the
partition structure at this vertex. The zero off-diagonal entries throughout the
complex reflect the perpendicular nature of axis-aligned splits, a geometric
signature of standard decision trees. For such axis-aligned partitions,
condition numbers directly measure the ratios of boundary lengths (in 2D) or
areas (in 3D), providing a clear geometric interpretation of partition
regularity.

\subsubsection*{2.2.3 Ensemble Geometric Structures}

When multiple decision trees partition the same feature space, how should we
reconcile their different geometric views? Consider two trees in a random forest
that split the same region differently. One tree might place a boundary through
the middle of a cluster, while another respects the cluster's natural boundary.
Some cells appear in nearly every tree, suggesting stable geometric structures,
while others appear rarely, indicating disagreement about the appropriate
partition.

This variation across trees contains valuable information. Regions where trees
consistently agree likely represent robust patterns in the data, while regions
of disagreement may indicate uncertainty or complexity. By tracking how often
cells from different trees coincide or share boundaries, we can build a richer
geometric picture than any single tree provides.

Tree ensembles $\mathcal{E} = \{T_1, T_2, \ldots, T_B\}$ create multiple
overlapping partitions of the feature space. Each tree $T_b$ induces its own
partition $\mathcal{P}_{T_b}$. The refined partition approach overlays all
individual tree partitions:
$$\mathcal{P}_{\mathcal{E}} = \left\{\bigcap_{b=1}^B C_{b,\alpha_b} : C_{b,\alpha_b} \in \mathcal{P}_{T_b}\right\} \setminus \{\emptyset\}$$

We now construct a Riemannian structure that captures both the geometry of this
refined partition and the agreement patterns between trees in the ensemble.

\textbf{The 0-chain structure.} We begin with the simplest case: inner products
on vertices. Each vertex $v_i$ in the nerve complex corresponds to a cell $C_i$
in the refined partition. Following our principle that cell volume reflects
importance, we define:
$$
\langle v_i, v_i \rangle_{v_i,0}^{\text{ens}} = \text{vol}_n(C_i) + \lambda_0 \cdot \text{Freq}(C_i)
$$
where
$$
\text{Freq}(C_i) = \frac{1}{B}\sum_{b=1}^B \mathbf{1}[C_i \subseteq \text{some cell of tree } b]
$$
measures how many trees contribute to defining this cell.

\textbf{The 1-chain structure.} For edges, we must capture how the ensemble
views adjacency relationships. Following the kernel perspective on random
forests, we define the co-occurrence frequency between two cells:
$$
K(C_i, C_j) = \frac{1}{B} \sum_{b=1}^B \mathbf{1}[\text{cells } C_{b,\alpha_b^{(i)}} \text{ and } C_{b,\alpha_b^{(j)}} \text{ are adjacent in tree } b]
$$
This quantity, ranging from 0 to 1, measures how consistently the ensemble
treats these cells as neighbors.

To create a Riemannian structure that captures both the refined partition's
geometry and the ensemble's confidence in that geometry, we enhance the standard
geometric inner products with ensemble-based terms. The geometric term captures
the physical configuration (boundary sizes and angles), while the ensemble term
adds information about statistical stability across the ensemble. We combine
them additively so that edges can have non-zero inner product even when they
meet at right angles (where the geometric term vanishes), allowing the ensemble
agreement to maintain connectivity in the Riemannian structure. The ensemble
term takes a similar functional form to the geometric term to ensure comparable
scaling.

For edges $e = v_i \wedge v_j$ and $e' = v_i \wedge v_k$ in
$\overline{\text{Star}}(v_i)$, we combine geometric information with ensemble
agreement:
$$
\langle e, e' \rangle_{v_i,1}^{\text{ens}} = \sqrt{\text{vol}_{n-1}(F_{ij}) \cdot \text{vol}_{n-1}(F_{ik})} \cdot \cos\theta_{jk} + \lambda_1 \sqrt{K(C_i, C_j) \cdot K(C_i, C_k)}
$$
where $F_{ij} = \overline{C_i} \cap \overline{C_j}$ and $\lambda_1 > 0$ weights the ensemble term.

For axis-aligned tree partitions, $\cos\theta_{jk} = 0$ for $j \neq k$, yielding:
$$
\langle e, e' \rangle_{v_i,1}^{\text{ens}} = \begin{cases}
  \text{vol}_{n-1}(F_{ij}) + \lambda_1 K(C_i, C_j) & \text{if } e = e' \\
  \lambda_1 \sqrt{K(C_i, C_j) \cdot K(C_i, C_k)} & \text{if } e \neq e'
\end{cases}
$$
The off-diagonal terms, arising purely from ensemble agreement, quantify how similarly the trees view different adjacency relationships.

\textbf{The general p-chain structure.} The construction extends naturally to higher dimensions. For a collection of cells $\{C_{i_1}, C_{i_2}, \ldots, C_{i_{p+1}}\}$, we generalize the co-occurrence frequency:
$$
K(C_{i_1}, \ldots, C_{i_{p+1}}) = \frac{1}{B} \sum_{b=1}^B \mathbf{1}[\text{cells } C_{b,\alpha_b^{(i_1)}}, \ldots, C_{b,\alpha_b^{(i_{p+1})}} \text{ form a } p\text{-simplex in tree } b]
$$
This measures how often these $p+1$ cells simultaneously meet in individual trees.

For $p$-simplices $\rho = \sigma \wedge v_1 \wedge \cdots \wedge v_q$ and $\rho' = \sigma \wedge v'_1 \wedge \cdots \wedge v'_q$ in $\Lambda_p(\overline{\text{Star}}(\sigma)))$, we define:
$$
\langle \rho, \rho' \rangle_{\sigma,p}^{\text{ens}} = \langle \rho, \rho' \rangle_{\sigma,p}^{\text{geom}} + \lambda_p \cdot K_p(\rho, \rho')
$$
where the geometric term follows Section 2.2.2, and the ensemble term uses:
$$
K_p(\rho, \rho') = \det(K_{\rho,\rho'})
$$
The Gram matrix $K_{\rho,\rho'}$ has entries:
$$
[K_{\rho,\rho'}]_{ij} = K(\text{vertices of } \sigma \cup \{v_i\}, \text{vertices of } \sigma \cup \{v'_j\})
$$
More explicitly, if $\sigma$ has vertices $\{u_1, \ldots, u_k\}$, then:
$$
[K_{\rho,\rho'}]_{ij} = K(u_1, \ldots, u_k, v_i, u_1, \ldots, u_k, v'_j)
$$

For example, if $\sigma = v_0$ (a vertex), $\rho = v_0 \wedge v_1 \wedge v_2$
and $\rho' = v_0 \wedge v_3 \wedge v_4$, then:
$$
K_{\rho,\rho'} = \begin{pmatrix}
K(v_0, v_1, v_3) & K(v_0, v_1, v_4) \\
K(v_0, v_2, v_3) & K(v_0, v_2, v_4)
\end{pmatrix}
$$
Each entry measures how often the corresponding triple of cells forms a
2-simplex across the ensemble.

\textbf{Choice of parameters.} The weights $\lambda_p$ balance geometric and
ensemble contributions at each dimension. A natural choice ensures both terms
have comparable scales:
$$
\lambda_p = \bar{v}_p = \frac{1}{|K_p|} \sum_{\sigma \in K_p} \text{vol}_{n-p}(F_\sigma)
$$
where $|K_p|$ denotes the number of $p$-simplices and $F_\sigma$ is the face corresponding to simplex $\sigma$.

This construction reveals how ensemble methods naturally incorporate uncertainty
into their geometric structure. High co-occurrence frequencies indicate stable
geometric relationships that persist across the ensemble, while low frequencies
identify regions where trees disagree about local structure. The Riemannian
framework transforms these agreement patterns into geometric quantities that can
guide refinement and regularization.

\subsubsection*{2.2.4 Gradient Boosting as Geometric Flow}

What happens to the geometry of partitions as gradient boosting sequentially
adds trees? Unlike random forests where all trees exist simultaneously, gradient
boosting builds models incrementally, with each new tree modifying the
ensemble's partition structure. Early trees might create broad regions capturing
global patterns, while later trees add fine-grained cells to correct local
errors. This sequential refinement suggests a dynamic process, but can we
characterize it mathematically as a geometric flow?

Understanding this evolution could unlock new insights into gradient boosting's
effectiveness and limitations. Does the partition geometry stabilize as training
progresses, or does it become increasingly fragmented? When should we stop
adding trees based on geometric rather than purely predictive criteria? By
viewing gradient boosting through the lens of evolving geometric structures, we
can develop principled approaches to model selection and regularization that go
beyond monitoring training error.

Gradient boosting exhibits a natural geometric interpretation when viewed
through the lens of evolving Riemannian structures. Unlike random forests where
all trees contribute simultaneously, gradient boosting sequentially adds trees,
creating a dynamic evolution of the partition geometry. At iteration $m$, the
ensemble consists of trees $T_1, T_2, \ldots, T_m$ with the refined partition:
$$\mathcal{P}^{(m)} = \left\{\bigcap_{i=1}^m C_{i,k_i} : C_{i,k_i} \in \mathcal{P}_{T_i}\right\} \setminus \{\emptyset\}$$

As we add trees, both the partition structure and the co-occurrence frequencies
evolve. Each cell in the refined partition at iteration $m$ has the form
$C_i^{(m)} = \bigcap_{b=1}^m C_{b,\alpha_b^{(i)}}$ where
$C_{b,\alpha_b^{(i)}} \in \mathcal{P}_{T_b}$ is the cell from tree $b$ that
participates in forming $C_i^{(m)}$.

The co-occurrence frequency at iteration $m$ for cells $C_i^{(m)}$ and $C_j^{(m)}$ is:
$$
K^{(m)}(C_i^{(m)}, C_j^{(m)}) = \frac{1}{m} \sum_{b=1}^m \mathbf{1}[\text{cells } C_{b,\alpha_b^{(i)}} \text{ and } C_{b,\alpha_b^{(j)}} \text{ are adjacent in tree } b]
$$
When tree $T_{m+1}$ is added, the refined partition may change. A cell
$C_i^{(m)}$ may be split into multiple cells in $\mathcal{P}^{(m+1)}$, or it may
remain intact if $C_{m+1,\alpha_{m+1}^{(i)}}$ (the corresponding cell in tree
$T_{m+1}$) fully contains it. For cells $C_i^{(m+1)}$ and $C_j^{(m+1)}$ in the
new refined partition:
$$
K^{(m+1)}(C_i^{(m+1)}, C_j^{(m+1)}) = \frac{1}{m+1} \sum_{b=1}^{m+1} \mathbf{1}[\text{cells } C_{b,\alpha_b^{(i)}} \text{ and } C_{b,\alpha_b^{(j)}} \text{ are adjacent in tree } b]
$$

For the Riemannian structure, following Section 2.2.3, edges $e = v_i \wedge v_j$ and $e' = v_i \wedge v_k$ have inner product:
$$\langle e, e' \rangle_{v_i,1}^{(m)} = \sqrt{\text{vol}_{n-1}(F_{ij}^{(m)}) \cdot \text{vol}_{n-1}(F_{ik}^{(m)})} \cdot \cos\theta_{jk} + \lambda_1 \sqrt{K^{(m)}(C_i^{(m)}, C_j^{(m)}) \cdot K^{(m)}(C_i^{(m)}, C_k^{(m)})}$$

For axis-aligned trees where $\cos\theta_{jk} = 0$ for $j \neq k$, this simplifies to:
$$
\langle e, e' \rangle_{v_i,1}^{(m)} = \begin{cases}
  \text{vol}_{n-1}(F_{ij}^{(m)}) + \lambda_1 K^{(m)}(C_i^{(m)}, C_j^{(m)}) & \text{if } e = e' \\
  \lambda_1 \sqrt{K^{(m)}(C_i^{(m)}, C_j^{(m)}) \cdot K^{(m)}(C_i^{(m)}, C_k^{(m)})} & \text{if } e \neq e'
\end{cases}
$$

The evolution of the inner product structure reveals how boosting modifies geometry:
$$\langle e, e' \rangle_{v_i,1}^{(m+1)} - \langle e, e' \rangle_{v_i,1}^{(m)} = \Delta_{\text{geom}} + \Delta_{\text{ens}}$$

For the diagonal case ($e = e' = v_i \wedge v_j$):
$$\Delta_{\text{geom}} = \text{vol}_{n-1}(F_{ij}^{(m+1)}) - \text{vol}_{n-1}(F_{ij}^{(m)})$$
$$\Delta_{\text{ens}} = \lambda_1 \left[K^{(m+1)}(C_i^{(m)}, C_j^{(m)}) - K^{(m)}(C_i^{(m)}, C_j^{(m)})\right] = \frac{\lambda_1}{m+1}\left[\mathbf{1}[\text{adjacent in } T_{m+1}] - K^{(m)}(C_i^{(m)}, C_j^{(m)})\right]$$

For the off-diagonal case ($e = v_i \wedge v_j$, $e' = v_i \wedge v_k$ with $j \neq k$):
$$\Delta_{\text{geom}} = 0 \quad \text{(for axis-aligned trees)}$$
$$\Delta_{\text{ens}} = \lambda_1 \left[\sqrt{K^{(m+1)}(C_i^{(m)}, C_j^{(m)}) \cdot K^{(m+1)}(C_i^{(m)}, C_k^{(m)})} - \sqrt{K^{(m)}(C_i^{(m)}, C_j^{(m)}) \cdot K^{(m)}(C_i^{(m)}, C_k^{(m)})}\right]$$

The geometric change $\Delta_{\text{geom}}$ is non-zero only when tree $T_{m+1}$
splits one or both of the cells $C_i^{(m)}$, $C_j^{(m)}$, creating a new or
modified boundary. The ensemble change $\Delta_{\text{ens}}$ always occurs and
can be positive (strengthening the connection) or negative (weakening it)
depending on whether the new tree agrees with the existing ensemble structure.

The gradient boosting process thus traces a path through the space of Riemannian
simplicial complexes, with each tree refining both the combinatorial structure
(through new cells) and the geometric structure (through modified co-occurrence
patterns). This dual evolution provides rich information about the learning
dynamics, suggesting that geometric properties could guide both tree construction
and stopping criteria. We explore these regularization possibilities in detail
in the next section.

\subsubsection*{2.2.5 Geometric Regularization}

How can we translate geometric insights into practical regularization strategies
that improve model performance? Traditional regularization penalizes the number
of parameters or tree depth, but these indirect constraints fail to address the
spatial pathologies we have identified. A model might satisfy depth constraints
while creating cells with extreme volume disparities, where tiny regions overfit
to individual points while vast empty spaces remain in adjacent cells. Another
might generate boundaries meeting at knife-sharp angles that will likely
misclassify nearby test points. We need regularization that directly targets
these geometric problems.

What would ideal geometric regularization achieve? It should penalize extreme
volume ratios between cells, encouraging balanced partitions. It should
discourage sharp angles where cells meet, promoting smoother decision
boundaries. For ensemble methods, it should reward geometric
consensus—configurations where multiple models agree on stable boundary
locations. By incorporating these geometric criteria into model training, we can
guide learning algorithms toward partitions that not only fit the data but also
exhibit the spatial regularity associated with good generalization.

The Riemannian structure on partition complexes enables the formulation of
geometric regularization terms that complement classical approaches from Section
2.1.2. Rather than penalizing only the number of cells or the smoothness of
functions within them, we can now formulate regularization terms that directly
target spatial pathologies.

A key insight for high-dimensional applications is that real data typically
exhibits low intrinsic dimensionality despite residing in high-dimensional
feature spaces. Rather than computing geometric properties in the full ambient
space, implementations should leverage local dimensionality reduction within
each partition cell. By projecting to subspaces that retain 95-99\% of local
variance, we can ensure that geometric computations remain both meaningful and
tractable. For instance, data in $\mathbb{R}^{1000}$ might have an intrinsic
dimensionality of approximately 10, reducing boundary volume computations from
intractable high-dimensional operations to manageable calculations in this
lower-dimensional representation. This dimensionality reduction approach
provides a practical adaptation of our geometric framework to the realities of
high-dimensional data.

With the geometric framework in place, penalties can explicitly depend on
spatial properties of the partition. Natural geometric quantities include cell
volumes $\text{vol}_{n}(C_\alpha)$, surface areas of partition boundaries
$\text{vol}_{n-1}(\partial C_\alpha \cap \partial C_\beta)$, and angles between
adjacent cells at their boundaries. Such penalties encourage balanced partitions
with regular cell shapes and smooth transitions between regions. For instance, a
geometric regularization term takes the form
$$
R_{\text{geom}} = \sum_{\alpha} \phi(\text{vol}_n(C_\alpha)) + \sum_{\langle \alpha,\beta \rangle} \psi(\text{vol}_{n-1}(\partial C_\alpha \cap \partial C_\beta), \theta_{\alpha\beta}),
$$
where $\phi$ penalizes extreme volumes, $\psi$ incorporates both boundary size
and the angle $\theta_{\alpha\beta}$ between cells, and the second sum runs over
adjacent cell pairs.

As we will demonstrate, the Riemannian structure on partition-induced simplicial
complexes provides a principled framework for computing and incorporating these
geometric properties into regularization schemes. The simplicial complex
naturally encodes adjacency relationships, while the Riemannian metric captures
local geometric information including angles and volumes. This geometric
perspective not only suggests alternative regularization strategies but also
provides tools for analyzing existing methods through their implicit geometric
biases.

\paragraph{Geometric Quantities and Their Statistical Interpretation} The
Riemannian structure enables computation of several geometric quantities with
direct statistical relevance. Cell volumes, encoded in the diagonal entries of
Gram matrices, indicate the "influence region" of each local model. Extreme
volume disparities may signal overfitting, where the model creates tiny cells to
accommodate outliers or noise. The condition numbers of vertex Gram matrices
$\kappa(G_v)$ measure local anisotropy, with large values indicating stretched
or poorly shaped cells that may not generalize well to new data.

Boundary properties provide complementary information. The total surface area of
a cell's boundary, $\sum_{j \sim i} \text{vol}_{n-1}(F_{ij})$ where the sum runs
over adjacent cells, quantifies the cell's complexity. Cells with large
surface-to-volume ratios may indicate excessive fragmentation of the feature
space. Dihedral angles between adjacent cells, extracted from the off-diagonal
entries of Gram matrices, reveal the smoothness of decision boundaries. Sharp
angles (small $\cos\theta$) suggest abrupt transitions that may indicate
overfitting to training data patterns.

Higher-order simplices capture multi-way interactions between cells. The volume
of a $p$-simplex, computed as $\sqrt{\det(G_\sigma)}$ for the Gram matrix on
$\Lambda_p(\overline{\text{Star}}(\sigma))$, measures the "strength" of the
$(p+1)$-way intersection. Large volumes for high-dimensional simplices indicate
complex intersection patterns that may be difficult to interpret or may result
from overfitting.

\paragraph{General Formulation of Geometric Penalties} Building on these quantities, we
propose geometric regularization terms that can be incorporated into the
training objectives of partition models. A general geometric regularization term
takes the form:
$$
R_{\text{geom}}(\mathcal{P}) = \sum_{k=0}^{\dim(K)} \lambda_k \sum_{\sigma \in K_k} \psi_k(\sigma)
$$
where $\psi_k$ are penalty functions operating on $k$-simplices and
$\lambda_k \geq 0$ are regularization strengths for each dimension.

For vertices (0-simplices), natural penalties include:
$$
\psi_0(v_i) = \phi\left(\frac{\text{vol}_{n}(C_i)}{\text{vol}_{n}(\mathcal{D})/|K_0|}\right) + \gamma \log(\kappa(G_{v_i}))
$$
where $\phi$ is a convex function penalizing deviations from the average cell
volume (e.g., $\phi(t) = t + 1/t$ for $t > 0$), and the second term penalizes
anisotropic local geometry. This encourages balanced partitions with
well-conditioned cells.

For edges (1-simplices), we can penalize both boundary complexity and sharp angles:
$$
\psi_1(v_i \wedge v_j) = \text{vol}_{n-1}(F_{ij})^{\alpha} \cdot h(\theta_{ij})
$$
where $\alpha > 1$ penalizes large boundaries and $h$ penalizes extreme angles, e.g.,
$$
h(\theta) = (\pi/2 - |\theta - \pi/2|)^2
$$
favors perpendicular boundaries. This encourages partitions with moderate boundary sizes and avoids both very sharp and very shallow angles.

For higher-dimensional simplices, penalties can target complex intersection patterns:
$$
\psi_p(\sigma) = \left(\frac{\text{vol}_p(\sigma)}{\text{vol}_p^{\text{ref}}}\right)^\beta \cdot \mathbf{1}[p > p_0]
$$
where $\text{vol}_p^{\text{ref}}$ is a reference volume scale, $\beta > 0$ controls the penalty strength, and $p_0$ is a threshold beyond which high-order intersections are discouraged. This limits the combinatorial complexity of the partition.

\paragraph{Implementation in Partition Models} These penalties can be
incorporated into specific partition models as follows. For decision trees, the
splitting criterion becomes:
$$
\text{Score}(s) = \text{ImpurityReduction}(s) - \eta \cdot \Delta R_{\text{geom}}(s)
$$
where $\Delta R_{\text{geom}}(s)$ is the change in geometric penalty from split
$s$. This can be computed efficiently by maintaining local geometric information
during tree construction.

In MARS models, geometric penalties augment the forward stepwise procedure. When
considering adding a new basis function, we evaluate both the reduction in
residual sum of squares and the induced change in partition geometry. The
geometric penalty can favor basis functions that create more regular partitions,
potentially improving model interpretability and generalization.

For random forests and other ensemble methods, geometric regularization can
guide both individual model training and ensemble aggregation. Individual trees
can be regularized to avoid extreme geometries, while the ensemble combination
weights can be adjusted to favor geometrically compatible partitions. The
co-occurrence frequencies $K(C_i, C_j)$ provide a natural measure of geometric
compatibility across trees.

\paragraph{Geometric Regularization for Gradient Boosting} In gradient boosting,
the sequential nature of gradient boosting offers unique regularization
opportunities through the geometric flow perspective developed in Section 2.2.4.
We propose three complementary approaches:

The geometric flow perspective suggests the following monitoring criterion.
Define the geometric energy at iteration $m$ as:
$$
E^{(m)} = \sum_{v \in K_0^{(m)}} \log \kappa(G_v^{(m)})
$$
where $G_v^{(m)}$ is the Gram matrix at vertex $v$ using the current Riemannian
structure. We hypothesize that this geometric energy measure could potentially
track the evolution of partition regularity during boosting. The rate of energy
change:
$$\frac{E^{(m+1)} - E^{(m)}}{E^{(m)}} > \epsilon$$
provides a candidate geometric early stopping criterion that warrants empirical
investigation.

\textbf{Regularized tree construction.} When fitting tree $T_{m+1}$ in gradient
boosting with general loss function $\ell(y, \hat{y})$, incorporate geometric
objectives:
$$
T_{m+1} = \arg\min_T \left\{ \mathcal{L}_{\text{residual}}(T) + \gamma \cdot \mathcal{R}_{\text{geom}}(T | \mathcal{P}^{(m)}) \right\}
$$
where
$$
\mathcal{L}_{\text{residual}}(T) = \sum_{i=1}^N \ell(y_i, F^{(m)}(x_i) + \eta f_T(x_i)),
$$
$\eta$ is the learning rate and $f_T(x_i)$ the prediction of tree $T$ at point
$x_i$. The geometric regularizer favors trees that maintain high co-occurrence
with existing structure:
$$\mathcal{R}_{\text{geom}}(T | \mathcal{P}^{(m)}) = -\sum_{(i,j) \in \text{adjacent pairs in } T} K^{(m)}(C_i \cap \text{leaf}_T, C_j \cap \text{leaf}_T)$$
This encourages new trees to respect boundaries where previous trees agree, reducing geometric complexity.

\textbf{Weighted co-occurrence.} In standard gradient boosting with learning
rate $\eta$, predictions combine as $F^{(m)}(x) = \sum_{i=1}^m \eta f_{T_i}(x)$.
This suggests using weighted co-occurrence frequencies in the Riemannian
structure:
$$
K_{\eta}^{(m)}(C_i^{(m)}, C_j^{(m)}) = \frac{\sum_{b=1}^m \eta^b \mathbf{1}[\text{cells } C_{b,\alpha_b^{(i)}} \text{ and } C_{b,\alpha_b^{(j)}} \text{ are adjacent in tree } b]}{\sum_{b=1}^m \eta^b}
$$
For $\eta < 1$, later trees contribute less to the geometric structure,
reflecting their role in fine-tuning rather than establishing primary
boundaries.

\textbf{Spectral monitoring.} The evolving Riemannian structure induces a
sequence of Hodge Laplacian operators $\{\Delta_p^{(m)}\}_{p=0}^{\dim K}$ on the
nerve complex of partition $\mathcal{P}^{(m)}$. These operators encode geometric
properties at multiple scales, providing rich information about the evolution of
partition structure during boosting.

The graph Laplacian on vertices is:
$$\Delta_0^{(m)} = D^{(m)} - A^{(m)}$$
where $D^{(m)}$ is the diagonal degree matrix and $A^{(m)}$ is the weighted adjacency matrix with entries:
$$A_{ij}^{(m)} = \begin{cases}
\text{vol}_{n-1}(F_{ij}^{(m)}) + \lambda_1 K^{(m)}(C_i^{(m)}, C_j^{(m)}) & \text{if } C_i^{(m)} \text{ and } C_j^{(m)} \text{ are adjacent} \\
0 & \text{otherwise}
\end{cases}$$

The eigenvalues $0 = \lambda_1^{(m)} \leq \lambda_2^{(m)} \leq \cdots$ encode
global connectivity. The spectral gap $\lambda_2^{(m)}$ measures how
well-connected the partition is, with large gaps indicating strong connectivity
and coherent geometric structure, while small gaps suggest near-disconnection
through fragmentation or geometric bottlenecks.

Higher-order Hodge Laplacians
$\Delta_p^{(m)} = \partial_{p+1} \partial_{p+1}^* + \partial_p^* \partial_p$ for
$p \geq 1$ capture multi-way interactions beyond pairwise relationships. The
edge Laplacian $\Delta_1^{(m)}$ reveals patterns in edge flows and boundary
cycles, while $\Delta_2^{(m)}$ captures the structure of triple cell
intersections in three or more dimensions. Higher-order Laplacians encode
increasingly complex multi-way meeting patterns, providing a complete
multi-scale view of the partition's geometric structure.

The complete spectral signature tracks geometric robustness across scales:
$$\text{SpectralSignature}^{(m)} = \left(\frac{\lambda_2^{(m)}}{\lambda_2^{(1)}}, \frac{\mu_{1,1}^{(m)}}{\mu_{1,1}^{(1)}}, \frac{\mu_{2,1}^{(m)}}{\mu_{2,1}^{(1)}}, \ldots\right)$$
where $\mu_{p,1}^{(m)}$ is the smallest positive eigenvalue of $\Delta_p^{(m)}$.

Decreasing ratios suggest loss of geometric robustness at the corresponding
scale. We propose monitoring:
$$\text{GeometricHealth}^{(m)} = \min_{p: \mu_{p,1}^{(1)} > 0} \frac{\mu_{p,1}^{(m)}}{\mu_{p,1}^{(1)}}$$

When this measure drops below a threshold, it may indicate the ensemble is
transitioning from learning robust structure to overfitting, though empirical
validation is needed to establish appropriate thresholds for specific
applications.

For computational efficiency, only the first few eigenvalues at each scale need
be computed using iterative methods.

\paragraph{Comparison with Classical Regularization} Classical machine learning
regularization and geometric regularization fundamentally differ in their
approach to controlling model complexity. Standard parameter regularization,
such as L1 or L2 penalties on weight vectors, controls model capacity globally
through constraints like $\|w\|_2^2$ but remains agnostic to the local geometric
structure of the induced partition. Similarly, structural regularization
approaches that limit tree depth or the number of parameters impose uniform
complexity constraints without considering the quality or configuration of the
resulting partition. In contrast, geometric regularization directly targets
specific spatial pathologies such as cells with extreme volume disparities,
boundaries meeting at sharp angles, or surface-area-to-volume ratios that
indicate excessive fragmentation. While standard regularization may indirectly
influence partition geometry through its constraints on model parameters,
geometric regularization provides direct control over these spatial properties.
This direct control over geometric structure offers the potential for models
that not only generalize better but also exhibit more interpretable decision
boundaries and more balanced spatial decompositions of the feature space.

These geometric regularization strategies apply the mathematical structures of
Sections 2.2.1-2.2.4 to partition model training. By penalizing geometric
irregularities directly, this approach addresses limitations of classical
regularization approaches that focus solely on model complexity or function
smoothness. The result is a principled framework for creating partition models
with better generalization properties and more interpretable geometric
structures.

\subsection*{2.3 Refinement of Partition Models}

The geometric framework developed in Section 2.2 can be extended to address the
refinement of piecewise constant predictions in partition models. A decision
tree predicts 0.2 in one cell and 0.8 in an adjacent cell, creating an abrupt
jump at their boundary. Should we accept this discontinuity, or is there
information in the geometric structure that could guide smoothing? When three
cells meet at a vertex with values 0.2, 0.8, and 0.5, how should their angular
configuration influence interpolation?

This challenge reflects a fundamental limitation: while partition models produce
piecewise constant predictions, they often approximate smooth phenomena. A tree
modeling house prices creates sharp boundaries between neighborhoods, yet prices
vary continuously across space. Standard graph-based smoothing methods address
this by updating values using weighted averages along edges. However, these
methods treat each edge independently. When three cells meet at a vertex, graph
Laplacian smoothing cannot distinguish whether they meet symmetrically or form a
narrow wedge, potentially smoothing inappropriately.

The importance of higher-order geometric relationships becomes clear in
structured prediction tasks. In molecular property prediction, aromatic rings
fundamentally alter chemical behavior through resonance effects that cannot be
decomposed into pairwise bond properties. A model that smooths predictions based
only on pairwise atomic connections would miss these collective effects,
potentially leading to systematic errors for aromatic compounds. Similarly, in
computer vision tasks on mesh data, the curvature at a vertex, determined by how
surrounding triangles meet, often correlates strongly with semantic boundaries.
These examples demonstrate that multi-way geometric relationships contain
essential information that pairwise methods cannot capture.

This section develops refinement techniques that exploit the full geometric
structure of partition models. We extend classical smoothing operators to
respect not just adjacency but the complete simplicial geometry, incorporating
information from triangles, tetrahedra, and higher-order structures. The key
insight is that smoothing should respect not just which cells are adjacent, but
how they meet in space.

We construct the extended Laplacian through a careful lifting procedure. Given a
function $u$ defined on vertices of our simplicial complex $K$, we must first
extend it to higher-dimensional simplices before applying appropriate smoothness
operators. The lifting process reveals a deep connection between discrete and
continuous mathematics.

The Whitney form method provides an elegant approach to lifting vertex functions
to higher-dimensional simplices. For each p-simplex $\sigma = [v_0, \ldots, v_p]$,
Whitney introduced a canonical p-form
$$\omega_\sigma = p! \sum_{i=0}^p (-1)^i \lambda_i \, d\lambda_0 \wedge \cdots \wedge \widehat{d\lambda_i} \wedge \cdots \wedge d\lambda_p$$
where $(\lambda_0, \ldots, \lambda_p)$ are barycentric coordinates (satisfying
$\sum_i \lambda_i = 1$ and $\lambda_i \geq 0$) and the hat denotes omission.
This form has the remarkable property that
$$\int_\tau \omega_\sigma = \begin{cases} 1 & \text{if } \tau = \sigma \\ 0 & \text{if } \tau \neq \sigma \end{cases}$$

To lift a vertex function $u$ to a p-cochain using Whitney forms, we evaluate:
$$\kappa_p^{\text{Whitney}}(u)(\sigma) = \sum_{i=0}^p u(v_i) \int_\sigma \lambda_i \, d\lambda_0 \wedge \cdots \wedge \widehat{d\lambda_i} \wedge \cdots \wedge d\lambda_p$$
This formula interpolates the vertex values $u(v_i)$ using integrals of the
barycentric coordinates against the appropriate volume forms.

An alternative lifting approach uses harmonic extension, seeking the function on
each simplex that minimizes the Dirichlet energy while matching prescribed
vertex values. For each p-simplex $\sigma$, we solve Laplace's equation
$\Delta H = 0$ with the vertex values of $u$ as boundary conditions. Among all
functions with these boundary values, the harmonic function has the smallest
integrated squared gradient $\int_\sigma \|\nabla H\|^2$, making it "smoothest"
in this specific energy-minimizing sense.

Remarkably, both the Whitney extension and harmonic extension yield the same
result: the lifted value on a p-simplex $\sigma$ is the barycentric average of
the vertex values, weighted by the p-dimensional volume of the simplex:
$$\kappa_p(u)(\sigma) = \frac{\text{vol}_p(\sigma)}{p+1} \sum_{i=0}^p u(v_i)$$
This formula shows that higher-dimensional simplices naturally encode averaged
information from their vertices (see Appendix~A for the complete derivation).

With lifting operators in hand, we can now define the extended Laplacian. Let
$L_p$ denote the Hodge Laplacian acting on p-cochains, defined through
$$L_p = \delta_{p-1}\delta_{p-1}^* + \delta_p^*\delta_p$$
where $\delta$ is the coboundary operator.

The extended Laplacian combines information from different dimensions:
$$L_{\text{ext}} u = \underbrace{L_0 u}_{\text{pairwise smoothing}} +
\underbrace{\alpha_1 \kappa_1^* L_1 \kappa_1 u}_{\text{loop consistency}} +
\underbrace{\alpha_2 \kappa_2^* L_2 \kappa_2 u}_{\text{triple interactions}} + \cdots$$
Each term enforces smoothness at a different scale:
\begin{itemize}[label=-, topsep=-5pt, partopsep=0pt, parsep=0pt, itemsep=0pt, after=\vspace{0.2cm}]
  \item $L_0 u$: Standard graph Laplacian smoothing along edges
  \item $\kappa_1^* L_1 \kappa_1 u$: Ensures gradients are consistent around loops
  \item $\kappa_2^* L_2 \kappa_2 u$: Regularizes curvature-like quantities
\end{itemize}
In general, the term $\kappa_p^* L_p \kappa_p$ incorporates p-dimensional
smoothness information. The operator $\kappa_p$ lifts vertex functions to
p-cochains using the Whitney/harmonic method, $L_p$ measures roughness in that
dimension, and $\kappa_p^*$ projects back to vertices. The parameters
$\alpha_p > 0$ control the relative influence of each dimension.

The extended Laplacian addresses the motivating problem by incorporating
multi-way cell interactions. When three cells meet at unusual angles, the
2-dimensional term $\kappa_2^* L_2 \kappa_2$ detects this through the geometry
encoded in the triangle they form. The resulting smoothing respects not just
adjacency but the full configuration of how cells meet.

A complementary variational approach emerges through simplicial splines. Given
observations $y = (y_1, \ldots, y_n)^T$ at vertices, we seek a function
minimizing the energy functional:
$$\mathcal{E}[u] = \underbrace{\|y - u\|^2\vphantom{\sum_{p=0}^{d}}}_{\text{data fidelity}} + \underbrace{\sum_{p=0}^{d}\sum_{k=1}^{K_p} \lambda_{pk} \langle \kappa_p u, L_p^{k} \kappa_p u \rangle_p}_{\text{multi-scale regularization}}$$

This functional balances fidelity to observed data against smoothness
constraints operating at multiple geometric scales. The regularization term
decomposes into contributions from different dimensions, each capturing distinct
aspects of smoothness. At the lowest dimension ($p=0$), the penalty directly
constrains differences between vertex values, encouraging nearby vertices to
have similar function values. The one-dimensional component ($p=1$) penalizes
variation along edges, analogous to controlling first derivatives in the
continuous setting. This term ensures that the discrete gradient field varies
smoothly throughout the complex.

Higher-dimensional terms encode increasingly sophisticated notions of
regularity. The two-dimensional component ($p=2$) measures variation across
triangles, capturing a discrete analogue of curvature by detecting how gradients
twist around vertices where multiple cells meet. In a concrete example of three
cells meeting at a vertex in a 2D mesh, this term identifies and penalizes
configurations where the gradients rotate sharply, preferring arrangements where
the function varies coherently across the junction. As the dimension $p$
increases, the regularization captures progressively more subtle patterns of
variation, with each term
$\lambda_{pk} \langle \kappa_p u, L_p^{k} \kappa_p u \rangle_p$ measuring
inconsistencies in how the function behaves around $(p+1)$-way intersections of
cells.

The parameters $\lambda_{pk}$ provide fine-grained control over the
regularization, allowing different weights for different dimensions $p$ and
different powers $k$ of the Laplacian operator. The use of powers $L_p^k$
enables control over higher-order smoothness within each dimension, analogous to
penalizing higher derivatives in classical spline theory. This multi-scale
framework naturally adapts to the geometric structure of the partition, applying
appropriate smoothness constraints based on how cells meet and interact within
the simplicial complex.

The minimizer of this functional satisfies the linear system:
$$\left(I + \sum_{p,k} \lambda_{pk} \kappa_p^* W_p L_p^{k-1} \kappa_p\right) u = y$$
This system, derived by setting the Fréchet derivative of $\mathcal{E}$ to zero,
has an intuitive structure despite its complex appearance. The operator on the
left decomposes as:
$$
\underbrace{I\vphantom{\sum_{p,k}}}_{\text{identity}} + \underbrace{\sum_{p,k} \lambda_{pk} \kappa_p^* W_p L_p^{k-1} \kappa_p}_{\text{weighted smoothing operators}}
$$
Each term in the sum implements a specific type of smoothing: the operator
$\kappa_p$ lifts the vertex function to p-dimensional structures, $L_p^{k-1}$
applies $(k-1)$-fold smoothing at that dimension, $W_p$ incorporates the
Riemannian metric weights, and $\kappa_p^*$ projects the smoothed result back to
vertices. The matrix $W_p$ encodes the full geometric structure of the complex.
For $p=1$, this includes not just edge lengths but also the inner products
between edges meeting at vertices, which involve both boundary areas and
dihedral angles. For $p=2$ and higher, $W_p$ similarly captures volumes and
relative orientations of higher-dimensional simplices. This rich geometric
encoding ensures that the smoothing respects not just the sizes of cells but
also their angular relationships in space.

To understand this concretely, consider the simplest case with only first-order edge smoothing ($p=1$, $k=1$):
$$\left(I + \lambda_{10} \kappa_1^* W_1 \kappa_1\right) u = y$$
Here, $\kappa_1 u$ computes discrete gradients on edges. The matrix $W_1$
incorporates the full Riemannian structure: diagonal entries contain edge
weights (derived from boundary areas), while off-diagonal entries encode angular
relationships through terms like
$$
\sqrt{\text{vol}_{n-1}(F_{ij}) \cdot \text{vol}_{n-1}(F_{ik})} \cos\theta_{jk}.
$$
The operator $\kappa_1^*$ then aggregates these geometrically-weighted gradients
back to vertices.

The higher-order terms ($L_p^{k-1}$ with $k > 1$) enable control over smoothness
derivatives. Just as classical splines can penalize not just function values but
also derivatives and higher-order variations, these terms penalize higher-order
roughness at each dimensional scale. The key advantage of this formulation is
that despite the sophisticated multi-scale smoothing it encodes, the final
system remains linear in the vertex values $u$, enabling efficient solution
using standard sparse linear solvers. The complexity in the formula directly
addresses our motivating challenge: incorporating geometric information beyond
simple adjacency to achieve smoothing that respects how cells meet in space.

Returning to our example of three cells meeting at a vertex, the extended
Laplacian and simplicial splines now provide geometrically-aware smoothing. If
the cells meet symmetrically, the smoothing preserves this symmetry. If two
cells form a narrow wedge against a third, the smoothing respects this
anisotropy. The higher-dimensional terms automatically adapt to the local
geometry, providing smoothing that graph-based methods cannot achieve.

These constructions demonstrate that smoothing operators can be systematically
extended to incorporate higher-order geometric relationships. The lifting
procedure provides a concrete mechanism for moving from vertex functions to
higher-dimensional structures where multi-way interactions become accessible to
standard differential operators. While the mathematical framework requires
careful development, the resulting operators admit efficient implementation
through linear systems on vertex functions.

The extended Laplacian and simplicial spline frameworks developed here offer one
approach to incorporating geometric information in the refinement of partition
models. The empirical validation of these refinement techniques across diverse
domains constitutes an important direction for future research. This geometric
perspective provides an alternative view where generalization might be related
to the geometric structures induced by partition models rather than solely to
model complexity.

\section*{2.4 Toward Density-Weighted Geometry and Curvature Perspectives}

Consider two decision trees trained on different datasets that happen to create
geometrically identical partitions of the ambient feature space. Both achieve
similar training accuracy, both induce identical simplicial complexes, and both
have identical Riemannian metrics encoding boundary areas and dihedral angles.
Yet one generalizes well to test data while the other overfits significantly.
What distinguishes them?

One possible explanation involves how the geometry aligns with the underlying
data distribution. A partition cell containing thousands of training points
carries different statistical weight than an equally-sized cell containing just
a few points. A boundary cutting through a dense cluster of data points has
different implications than one passing through empty space. The Riemannian
structure we have developed, while capturing rich geometric information, does
not explicitly account for this statistical structure.

This limitation becomes evident when analyzing failure modes of machine learning
models. Adversarial examples may exploit regions of low data density where
models must extrapolate. Overfitting can manifest as excessive geometric
complexity in sparsely populated regions. Transfer learning performance may
depend on whether source and target domains share similar density patterns along
with geometric features. These observations suggest that understanding machine
learning models may benefit from integrating geometric and statistical
perspectives.

How might we enrich our geometric framework to incorporate data density? One
natural approach draws inspiration from physics, where density-weighted metrics
appear in general relativity \cite{wald1984general} and fluid dynamics
\cite{arnold1966geometry}. In these contexts, the presence of matter or energy
warps the underlying geometry, creating a unified geometric-physical structure.
Similarly, we can modify our Riemannian metric on simplicial complexes to
reflect data density, making paths through high-density regions effectively
"shorter" than those through sparse regions.

This density weighting modifies the geometric analysis. Shortest paths now
follow data concentrations rather than geometric geodesics. The notion of
distance becomes statistical as well as geometric, opening the possibility of
defining curvature-like measures that capture how models adapt to sparse,
structured data.

This section develops these ideas systematically. We begin by incorporating
statistical structure through density-weighted metrics (Section 2.4.1). We then
develop discrete curvature measures that quantify how cell arrangements create
effective curvature (Section 2.4.2). Building on these foundations, we introduce
statistical Ricci curvature for edges of the simplicial complex, enabling a
Taylor-like regularization framework that penalizes both local vertex
irregularities and relational edge irregularities (Section 2.4.3). Finally, we
examine how these geometric structures evolve during training, providing
diagnostic tools for learning algorithms (Section 2.4.4).

The framework presented here is exploratory, offering new perspectives on
familiar phenomena. By revealing the geometry that emerges from the interaction
between sparse data and model partitions, we aim to contribute to the
understanding of generalization, robustness, and learning dynamics in machine
learning models. The geometric tools we develop raise intriguing questions:
Might different architectures induce characteristic curvature patterns that
explain their inductive biases and failure modes? How do geometric structures
evolve during training, and can this evolution predict generalization
performance? While these questions remain open, they suggest potentially
fruitful connections between discrete geometry and machine learning.

\subsubsection*{2.4.1 Incorporating Statistical Structure}

To incorporate data density into our geometric framework, we must first estimate
how data distributes across the simplicial complex. Each cell $C_i$ contains
some number of training points $n_i = |\{x_j : x_j \in C_i\}|$. These raw counts
provide a discrete, piecewise constant density estimate:
$\tilde{\rho}_i = n_i/\text{vol}(C_i)$. However, such estimates are unstable in
small cells and undefined in cells containing no data. We need a principled
approach to smooth these estimates while respecting the geometric structure. We
first show how to estimate density at vertices through a regularized
optimization that balances data fidelity with smoothness constraints. With
vertex densities established, we then face the challenge of extending these
values to edges and higher-dimensional simplices. Simply averaging vertex
densities ignores the rich geometric information encoded in our Riemannian
structure. We explore several principled alternatives that incorporate geometric
correction factors, each with different properties suited to particular
applications.

The simplicial spline framework from Section 2.3 offers one possible approach.
We seek a density function $\rho: V \to \mathbb{R}_+$ that minimizes:
$$\mathcal{E}[\rho] = \sum_{i \in V} \frac{(n_i - \text{vol}(C_i) \cdot \rho(v_i))^2}{\text{vol}(C_i)} + \lambda \sum_{p=0}^{d}\sum_{k=1}^{K_p} \langle \kappa_p \rho, L_p^{k} \kappa_p \rho \rangle_p$$
The first term ensures that density times volume approximates the observed
counts, while the regularization term promotes smoothness across the simplicial
complex. The weighting by $1/\text{vol}(C_i)$ gives equal importance to cells of
different sizes, preventing large cells from dominating the optimization. This
formulation naturally handles the common situation where some cells contain no
training points, interpolating density from neighboring cells through the
geometric structure.

The solution to this variational problem satisfies a linear system analogous to
that for simplicial splines:
$$\left(D + \lambda \sum_{p,k} \kappa_p^* W_p L_p^{k-1} \kappa_p\right) \rho = Dn$$
where $D$ is diagonal with $D_{ii} = 1/\text{vol}(C_i)$ and $n$ is the vector of
point counts. The multi-scale regularization through powers of the Hodge
Laplacian ensures that density estimates vary smoothly while respecting the
simplicial structure.

With vertex densities estimated, we must extend these values to edges and
higher-dimensional simplices. The simplest approach would average vertex
densities:
$$
\rho(e) = \frac{\rho(v_i) + \rho(v_j)}{2}
$$
for an edge $e = v_i \wedge v_j$. However, this ignores the geometric
information encoded in our Riemannian structure. We propose incorporating a
geometric correction factor:
$$\rho(e) = \frac{\rho(v_i) + \rho(v_j)}{2} \cdot \omega(e)$$
where
$$\omega(e) = \frac{\langle e, e \rangle^{1/2}}{\text{vol}_{n-1}(F_{ij})^{1/(n-1)}}$$

This correction factor is motivated by considering how edge lengths relate to
boundary geometry. The numerator $\langle e, e \rangle^{1/2}$ is the edge length
in our Riemannian structure, which by the compatibility conditions is the same
whether computed at vertex $v_i$ or $v_j$. The denominator
$\text{vol}_{n-1}(F_{ij})^{1/(n-1)}$ represents an "expected" length based on
the boundary area between cells. In regular geometries, when a polytope is
scaled by factor $\lambda$, edge lengths scale by $\lambda$ while
$(n-1)$-dimensional boundaries scale by $\lambda^{n-1}$, so edge length scales
as the $(n-1)$-th root of boundary volume. The ratio $\omega(e)$ thus compares
the actual edge length to this expected scale.

When $\omega(e) > 1$, the edge is effectively "stretched" in the Riemannian
metric relative to its boundary size, suggesting this connection carries more
geometric significance. We accordingly increase the density estimate along this
edge. Conversely, when $\omega(e) < 1$, the edge is "compressed," indicating
less geometric importance, and we decrease the density estimate. This approach
aims to incorporate both statistical and geometric structure in the density
interpolation.

The choice of this particular extension deserves scrutiny, as several
mathematically principled alternatives exist. The formula above uses arithmetic
averaging with a geometric correction factor, but other approaches may prove
more suitable depending on the application: The harmonic interpolation approach
would define edge density using the harmonic mean:
$$
\rho_{\text{harm}}(e) = \frac{2\rho(v_i)\rho(v_j)}{\rho(v_i) + \rho(v_j)} \cdot \omega(e)
$$
The harmonic mean naturally appears when considering conductances in parallel.
Recall that when two resistors with resistances $R_1$ and $R_2$ are connected in
parallel, the combined resistance $R$ is:
$$\frac{1}{R} = \frac{1}{R_1} + \frac{1}{R_2}.$$
Converting to conductances $G = 1/R$, this becomes simply:
$$G = G_1 + G_2$$
So conductances add in parallel. The combined resistance is:
$$R = \frac{R_1 R_2}{R_1 + R_2} = \frac{1}{\frac{1}{R_1} + \frac{1}{R_2}}.$$
This is half the harmonic mean of $R_1$ and $R_2$. This connection makes the
harmonic mean natural for resistor network analogies \cite{tetali1991random}.
This choice tends to be more sensitive to low-density regions, making paths
avoid sparse areas more strongly.

Alternatively, we could use the lifting operator framework more directly. Since
$\kappa_1(\rho)$ lifts the density function to edges, we might define:
$$
\rho_{\text{lift}}(e) = \kappa_1(\rho)(e) = \frac{\text{vol}_{n-1}(F_{ij})}{\text{vol}(C_i) + \text{vol}(C_j)} \cdot (\rho(v_i) + \rho(v_j))
$$
This approach weights the contribution of each vertex by its cell's relative
size, naturally handling cases where adjacent cells have very different volumes.

A geometric mean approach suggests:
$$
\rho_{\text{geom}}(e) = \sqrt{\rho(v_i) \cdot \rho(v_j)} \cdot \omega(e)
$$
This formulation has the advantage of being scale-invariant under simultaneous
scaling of densities, properties that may be desirable in certain contexts.

Each approach has different mathematical properties that may be relevant for the
resulting density-weighted metric. The arithmetic mean treats densities
linearly, the harmonic mean gives more weight to smaller values (as seen in the
resistor network analogy), the lifting-based approach explicitly incorporates
cell volumes into the weighting, and the geometric mean preserves multiplicative
relationships between densities. The appropriate choice likely depends on the
specific machine learning context and the phenomena being studied. This
multiplicity of mathematically principled definitions underscores the
exploratory nature of this framework and invites empirical investigation to
determine which formulation best captures the relationship between data density
and model behavior.

The incorporation of density fundamentally alters our notion of distance on the
simplicial complex. Drawing inspiration from optimal transport theory
\cite{villani2003topics, peyre2019computational} and information geometry
\cite{amari2016information}, we define a density-weighted metric:
$$d_\rho(e) = \frac{\ell(e)}{\rho(e)^\alpha}$$
where $\ell(e)$ is the original Riemannian length of edge $e$ and $\alpha \in (0,1]$ controls the strength of density weighting. This modification makes paths through high-density regions effectively shorter than those through sparse regions.

This modification of the metric has significant geometric consequences. Shortest
paths no longer follow geometric geodesics but instead curve toward data
concentrations. The distance between two vertices now depends not just on their
geometric separation but on the density of data along connecting paths. This
aligns with the intuition that model behavior in dense regions is more reliable
and should influence distant predictions more strongly than sparse regions.

Consider how this density weighting affects the analysis of our earlier example.
The boundary between cells $C_1$ (with 1,000 points) and $C_2$ (with 10 points)
now carries different weight than a boundary between two dense cells. Paths that
venture into the sparse cell $C_2$ become longer under the density-weighted
metric, reflecting the increased uncertainty in that region. The simplicial
complex retains its combinatorial structure, but the effective geometry warps to
reflect statistical confidence.

This density-weighted geometric framework provides the foundation for studying
how machine learning models adapt to data distributions. By enriching the
Riemannian structure with statistical information, we create a unified
geometric-statistical object that captures both the partition structure and its
relationship to data. The stage is now set for exploring how this enriched
geometry might exhibit forms of curvature that explain model behavior, a
possibility we explore in the following sections.

\subsubsection*{2.4.2 Vertex-based Curvatures for ML Models}

The density-weighted metric introduced in Section 2.4.1 opens the possibility of
defining curvature-like measures for machine learning models. Yet this
possibility immediately confronts us with a paradox: the partitions created by
these models divide Euclidean space into convex polytopes, each inheriting the
flat geometry of the ambient space. There is no intrinsic curvature in the
classical sense. Paths within cells follow straight lines, and parallel
transport around any loop returns vectors unchanged. How then can we speak
meaningfully of curvature?

The resolution lies in recognizing that real-world data, despite living in
high-dimensional Euclidean space, is typically sparse and highly structured.
This sparsity and non-uniform distribution effectively creates geometric
curvature through the arrangement of partition cells. When data concentrates in
specific regions or along lower-dimensional structures while leaving vast
regions empty, machine learning models must adapt their partitions to this
inhomogeneous structure. The resulting configuration of cells exhibits
curvature-like properties that emerge not from the geometry within cells, but
from how the cells are arranged and connected, much like how a geodesic dome
exhibits global curvature despite being built from flat triangular panels. We
seek to explore whether these curvature-like properties relate to model
behavior, generalization, and robustness through the curvature measures
developed below.

\paragraph{Data Geometry Based Vertex Curvatures}

We begin by developing curvature measures that capture how the arrangement of
cells responds to data density. These measures exploit the density-weighted
metric introduced in Section 2.4.1 to quantify various geometric phenomena at
vertices: how quickly neighborhoods grow, how geodesics spread, and how paths
navigate through the cell structure. Each measure captures a different aspect of
the effective geometry created by sparse data.

The most direct approach examines the growth rate of density-weighted balls.
Under the density-weighted metric, the distance between points depends on data
density along connecting paths. For a vertex $v$ in our simplicial complex,
define the density-weighted ball of radius $r$:
$$\mathcal{B}_r^\rho(v) = \{w \in V : d_\rho^{\text{geodesic}}(v,w) \leq r\}$$
where $d_\rho^{\text{geodesic}}$ is the shortest path distance using the
density-weighted metric.

In flat space with uniform density, we expect $|\mathcal{B}_r^\rho(v)|$ to grow
as $r^n$. Deviations from this growth rate indicate curvature. Define the
discrete ball-growth curvature at scale $r$ as:
$$\kappa_r^{\text{ball}}(v) = \frac{N_r(v) - N_r^{\text{flat}}(v)}{N_r^{\text{flat}}(v)}$$
where $N_r(v) = |\mathcal{B}_r^\rho(v)|$ is the actual number of vertices in the ball and $N_r^{\text{flat}}(v) = \rho(v) \cdot V_n \cdot r^n$ is the expected number in flat space with uniform density $\rho(v)$, where $V_n$ is the volume of the unit ball in $\mathbb{R}^n$.

This definition exhibits clear geometric meaning. When
$\kappa_r^{\text{ball}}(v) > 0$, the ball contains more vertices than expected,
suggesting positive curvature where geodesics emanating from $v$ converge back
together. When $\kappa_r^{\text{ball}}(v) < 0$, geodesics diverge more rapidly
than in flat space, indicating negative curvature. The scale parameter $r$
allows us to probe curvature at different resolutions, from local perturbations
to global structure.

A different approach to local curvature examines distance deviations in the
immediate neighborhood. Define the distance-deviation curvature as:
$$\kappa^{\text{dist}}(v) = 1 - \frac{\sum_{w \in \partial \overline{\text{Star}}(v)} d_\rho(v,w)}{|\partial \overline{\text{Star}}(v)| \cdot \bar{d}_\rho}$$
where $\partial \overline{\text{Star}}(v)$ denotes the set of vertices adjacent
to $v$ and $|\partial \overline{\text{Star}}(v)|$ is the number of such vertices
(the degree of $v$) and $\bar{d}_\rho$ is a characteristic distance scale in the
complex. This measures whether neighbors are systematically closer (positive
curvature) or farther (negative curvature) than typical.

Another geometric approach examines how geodesics spread as they emanate from a
vertex, adapting the notion of Jacobi fields to the discrete setting. In smooth
Riemannian geometry, the spreading of geodesics is governed by the Jacobi
equation, with positive curvature causing convergence and negative curvature
causing divergence \cite{do1992riemannian, jost2011riemannian}.

For our discrete setting, we introduce geodesic spray analysis to measure this
spreading. Consider all vertices at density-weighted distance $r$ from vertex
$v$, forming the discrete sphere:
$$\mathcal{S}_r^\rho(v) = \{w : d_\rho^{\text{geodesic}}(v,w) = r\}$$
For each $w \in \mathcal{S}_r^\rho(v)$, let $\gamma_{vw}$ be a shortest path
from $v$ to $w$. We measure geodesic spreading through two complementary
approaches.

The first measures angular spread at the source by examining initial directions:
$$\Theta_r^{\text{angle}}(v) = \frac{1}{|\mathcal{S}_r^\rho(v)|^2} \sum_{w,w' \in \mathcal{S}_r^\rho(v)} \angle_v(w,w')$$
where $\angle_v(w,w')$ is the angle between the first edges of paths
$\gamma_{vw}$ and $\gamma_{vw'}$, computed using the Riemannian inner product:
$$
\cos(\angle_v(w,w')) = \frac{\langle e_w, e_{w'} \rangle_v}{\|e_w\|_v \|e_{w'}\|_v}
$$

The second, more global measure examines how far apart geodesics end up after traveling distance $r$:
$$\Theta_r^{\text{spread}}(v) = \frac{\sum_{w,w' \in \mathcal{S}_r^\rho(v)} d_\rho^{\text{geodesic}}(w,w')}{|\mathcal{S}_r^\rho(v)|^2 \cdot 2r}$$

To convert these spreading measures to curvatures, we compare to expected values in flat space. Define the geodesic spreading curvature as:
$$\kappa_r^{\text{spray}}(v) = 1 - \Theta_r^{\text{spread}}(v)$$
When $\kappa_r^{\text{spray}}(v) > 0$, geodesics converge (positive curvature). When $\kappa_r^{\text{spray}}(v) < 0$, geodesics diverge more than expected (negative curvature).

A different local curvature measure examines triangular relationships in the immediate neighborhood. Define the triangle-based curvature as:
$$\kappa^{\text{tri}}(v) = \frac{1}{|\partial \overline{\text{Star}}(v)|^2} \sum_{w,w' \in \partial \overline{\text{Star}}(v)} \left(1 - \frac{d_\rho(w,w')}{d_\rho(v,w) + d_\rho(v,w')}\right)$$
This measures whether the triangle inequality is tight (zero curvature) or loose (positive curvature indicates neighbors are closer than the sum of distances through $v$).

While these spreading measures capture how geodesics diverge or converge, they
do not reveal the complexity of the paths themselves. A third approach examines
how geodesics navigate through the partition structure.

For vertices at distance $r$ from $v$, define the path complexity measure:
$$J_r(v) = \frac{1}{|\mathcal{S}_r^\rho(v)|} \sum_{w \in \mathcal{S}_r^\rho(v)} \frac{|\gamma_{vw}|_{\text{cells}}}{r}$$
where $|\gamma_{vw}|_{\text{cells}}$ counts the number of distinct cells the
path traverses. In flat space with uniform density, we expect paths to traverse
a number of cells proportional to their length, giving
$J_r(v) \approx 1/\bar{\ell}$ where $\bar{\ell}$ is the average cell diameter.

Define the path complexity curvature as:
$$\kappa_r^{\text{path}}(v) = J_r(v) \cdot \bar{\ell} - 1$$
When
$\kappa_r^{\text{path}}(v) > 0$, geodesics wind through more cells than
expected, indicating complex local structure or fragmented partitions. When
$\kappa_r^{\text{path}}(v) < 0$, paths traverse fewer cells, suggesting larger,
more regular cells in that region.

\paragraph{Functional Curvature at Vertices}

While the data based curvatures capture geometric structure of the underlying
partitions, machine learning models are ultimately function approximators and
understanding their behavior requires not just geometric but also functional
analysis. How can we extend curvature concepts to the functions defined on our
partition complexes?

Traditional regularization approaches penalize gradient magnitudes through L1 or
L2 norms, but these measures are inherently positive and miss the fundamental
signed nature of curvature. A function that oscillates wildly between convex and
concave behavior differs qualitatively from one that maintains consistent
convexity, yet gradient norms alone cannot distinguish these cases. This
limitation motivates the development of functional curvature measures that
capture both the magnitude and sign of local function behavior.

We propose several complementary approaches to functional curvature, each
capturing different aspects of how functions behave on simplicial complexes. The
most direct approach adapts the classical mean curvature through the graph
Laplacian. For a function $f$ defined on vertices, the discrete mean curvature
at vertex $v$ takes the form
$$\kappa_f^{\text{mean}}(v) = \frac{\Delta f(v)}{f(v) - \bar{f}_v}$$
where
$$
\Delta f(v) = \sum_{w \sim v} w_{vw}(f(w) - f(v))
$$
is the weighted graph Laplacian and
$$
\bar{f}_v = \frac{\sum_{w \sim v} w_{vw}f(w)}{\sum_{w \sim v} w_{vw}}
$$
represents the weighted average of $f$ over neighbors. This ratio captures
whether $f$ is locally convex (positive curvature) or concave (negative
curvature) relative to its neighborhood average.

A more geometric approach examines gradient variations around vertices through
an angle-based measure. For a vertex $v$, we define:
$$\kappa_f^{\text{angle}}(v) = \frac{1}{|E_v|(|E_v|-1)} \sum_{e_i, e_j \in E_v, i \neq j} \log\left(\frac{1 + \cos\angle(\vec{\nabla} f|_{e_i}, \vec{\nabla} f|_{e_j})}{1 - \cos\angle(\vec{\nabla} f|_{e_i}, \vec{\nabla} f|_{e_j})}\right)$$
where $E_v$ is the set of edges incident to $v$. For edge $e_i = v \wedge w_i$,
we define the gradient vector of $f$ along $e_i$ as:
$$\vec{\nabla} f|_{e_i} = \frac{f(w_i) - f(v)}{\ell(e_i)^2} \cdot e_i$$
where we scale by $\ell(e_i)^2$ to account for the edge having length $\ell(e_i)$. The cosine of the angle between gradient vectors is:
$$\cos\angle(\vec{\nabla} f|_{e_i}, \vec{\nabla} f|_{e_j}) = \text{sign}((f(w_i) - f(v))(f(w_j) - f(v))) \cdot \frac{\langle e_i, e_j \rangle_v}{\ell(e_i) \cdot \ell(e_j)}$$
This measure ranges from $-\infty$ to $+\infty$, with values near zero
indicating moderately varied gradient directions typical of regular points.
Large positive values occur when gradients are strongly aligned, suggesting a
local extremum where the function consistently increases or decreases in all
directions. Large negative values indicate opposing gradients characteristic of
saddle points, where the function increases in some directions while decreasing
in others. The logarithmic formulation amplifies the distinction between aligned
and opposing gradients, making it particularly sensitive to detecting extrema
and saddle configurations.

For functions with sufficient regularity, we can examine the geometry of level
sets through the discrete structure of the simplicial complex. The level set
curvature at vertex $v$ measures how the function deviates from linear behavior
by examining the variance of directional derivatives:
$$\kappa_f^{\text{level}}(v) = \frac{1}{|\partial \overline{\text{Star}}(v)|} \sum_{w \in \partial \overline{\text{Star}}(v)} \left(\frac{f(w) - f(v)}{d(v,w)} - \langle \nabla f \rangle_v \right)^2$$
where $\partial \overline{\text{Star}}(v)$ denotes the vertices in the star of $v$
excluding $v$ itself, $d(v,w)$ is the Riemannian distance between vertices, and
$$
\langle \nabla f \rangle_v = \frac{1}{|\partial \text{Star}(v)|} \sum_{w \in \partial \text{Star}(v)} \frac{f(w) - f(v)}{d(v,w)}
$$
is the average directional derivative. This measure captures how much the
function's rate of change varies in different directions emanating from $v$.
When $\kappa_f^{\text{level}}(v) = 0$, the function behaves linearly in the star
of $v$, while large values indicate that level sets bend significantly,
suggesting local nonlinearity or saddle-like behavior in the function landscape.

\paragraph{Vertex Curvatures as Aggregations}

A crucial insight emerges when we examine how these vertex curvatures relate to
edge-based information. The discrete mean curvature explicitly aggregates
differences along edges through the Laplacian operator. The angle-based
curvature aggregates angular information from pairs of edges. The level set
curvature averages squared deviations of edge-based gradients.

This observation reveals that vertex curvatures, while providing valuable local
diagnostics, inherently average away edge-specific information. A vertex might
exhibit zero mean curvature because all neighboring values equal the average, or
because positive and negative deviations cancel. The aggregation process loses
information about which specific edges contribute to the curvature and how.

This limitation becomes particularly relevant when designing regularization
strategies. Penalizing only vertex curvatures might miss pathological
configurations where problems concentrate along specific edges or boundaries.
Two models might have identical vertex curvature distributions while differing
dramatically in their edge-level behavior.

\paragraph{Connecting to Model Phenomena}

These discrete curvature measures may connect to machine learning phenomena in
several ways. We hypothesize that regions of positive geometric curvature, where
density-weighted balls grow faster than expected, could correspond to stable
predictions where the model's behavior remains consistent across neighborhoods.
Such regions might contain ample training data distributed uniformly,
potentially allowing the model to learn robust patterns.

Similarly, negative geometric curvature regions, where geodesics diverge
rapidly, might coincide with decision boundaries or areas of high model
variance. The diverging geodesics could reflect how small perturbations lead to
different predictions, which would be characteristic of regions where the model
lacks confidence.

The functional curvatures suggest additional connections. High positive mean
curvature might indicate local overfitting, where the model's predictions form
sharp peaks. Negative mean curvature could suggest underfitting, where the model
fails to capture local patterns. Non-zero angle-based curvature might identify
saddle points that could affect optimization or create vulnerabilities.

These potential connections between curvature measures and model behavior
motivate further empirical investigation. The scale-dependence of these measures
could prove particularly valuable for practical application. A vertex might
exhibit positive curvature at small scales, suggesting local stability, while
displaying negative curvature at larger scales, indicating its position within a
complex global structure. This multi-scale analysis helps distinguish between
local irregularities that might be smoothed and fundamental geometric features
that should be preserved.

Computing these discrete curvature measures requires only local calculations on
the simplicial complex, making them practical even for large models. The
vertex-based nature allows parallel computation and efficient updates as the
model evolves during training. These computational advantages, combined with
their diagnostic value, make discrete curvatures valuable tools for
understanding and improving machine learning models.

The framework developed in this section provides the foundation for more
sophisticated geometric analysis. While vertex curvatures offer important local
diagnostics, they aggregate edge-level information in ways that may obscure
important relational structure. This observation motivates the development of
edge-based curvature measures that can capture the fine-grained geometric
relationships between model regions, a topic we explore in the following
section.

\subsubsection*{2.4.3 Statistical Ricci Curvature and Geometric Regularization}

The vertex-based curvatures developed in Section 2.4.2 provide valuable local
diagnostics but inherently aggregate information from incident edges. This
aggregation, while useful for understanding local behavior, obscures the
fine-grained relational structure between cells that often determines model
performance. Just as Taylor series expansions capture function behavior through
derivatives of increasing order, we can construct regularization frameworks that
incorporate geometric information at multiple scales of the simplicial complex.

Consider a general geometric regularization functional that expands over the
dimensions of our simplicial complex:
$$R(f) = \sum_{p=0}^{\dim K} \lambda_p \sum_{\sigma \in K_p} \psi_p(\kappa_f^{\text{stat},p}(\sigma)),$$
where $\kappa_f^{\text{stat},p}(\sigma)$ represents a statistical curvature
measure for a $p$-simplex $\sigma$, and $\psi_p$ are penalty functions chosen
based on the desired properties at each dimension. The $p=0$ term captures
vertex-based curvatures as developed in Section 2.4.2. Higher-order terms
incorporate increasingly sophisticated relational information: edges ($p=1$)
capture pairwise relationships between cells, triangles ($p=2$) encode three-way
interactions, and so forth.

This Taylor-like expansion reveals a fundamental principle for geometric
regularization. Just as truncating a Taylor series at first order often provides
adequate approximation for smooth functions, considering only vertices and edges
may suffice for many machine learning applications. The vertex term ensures
local regularity, while the edge term captures the most important relational
structure. Higher-order terms become relevant when complex multi-way
interactions significantly impact model behavior.

For practical implementation, we focus on developing the $p=1$ (edge) term,
which provides the most important extension beyond vertex-based measures. This
leads to the regularization framework:
$$R(f) = \sum_{v \in V} \phi(|\kappa^{\text{stat}}_f(v)|) + \lambda \sum_{(v,w) \in E} \psi((\kappa_f^{\text{stat}}(v,w))^2)$$
The remainder of this section develops the edge-based statistical curvature
$\kappa_f^{\text{stat}}(v,w)$, showing how it captures relational properties
that vertex measures miss.

\paragraph{Edge-Based Statistical Curvature}

Following the decomposition principle established for vertex curvatures, we
define statistical Ricci curvature for an edge $e = v \wedge w$ as:
$$\text{Ric}_{\text{stat}}(v,w) = \text{Ric}_{\text{geom}}(v,w) + \text{Ric}_{\text{dens}}(v,w) + \text{Ric}_{\text{func}}(v,w)$$

This decomposition maintains the philosophy that statistical structure emerges
from the interplay of geometric configuration, data distribution, and functional
behavior, now applied to measuring relationships between cells rather than local
properties.

\paragraph{Geometric Ricci Curvature}

The geometric component adapts Ollivier's coarse Ricci curvature to our discrete
setting. This elegant construction measures how neighborhoods of adjacent
vertices relate to each other through optimal transport. For vertices $v$ and
$w$, we construct probability measures $\mu_v$ and $\mu_w$ supported on their
respective neighbors, with weights proportional to edge lengths in the
Riemannian metric:
$$\mu_v(u) = \frac{\langle v \wedge u, v \wedge u \rangle_{v}^{1/2}}{\sum_{u' \sim v} \langle v \wedge u', v \wedge u' \rangle_{v}^{1/2}}$$

The geometric Ricci curvature then measures how these probability distributions relate under optimal transport:
$$\text{Ric}_{\text{geom}}(v,w) = 1 - \frac{W_1(\mu_v, \mu_w)}{d_\rho(v,w)},$$
where $W_1$ denotes the Wasserstein-1 distance computed using our
density-weighted metric. The interpretation follows classical intuition:
positive values indicate that neighborhoods of $v$ and $w$ are more similar than
their distance would suggest, implying local geometric regularity. The mass from
$\mu_v$ can be transported to $\mu_w$ through paths shorter than expected,
suggesting convergence of geodesics. Negative values indicate geometric
divergence, where transporting mass requires longer paths than the direct
distance would suggest.

This construction naturally incorporates the full Riemannian structure of our
simplicial complex. The edge weights in the probability measures reflect
boundary areas and dihedral angles, while the density-weighted metric in the
Wasserstein distance accounts for data distribution. The result is a
geometrically and statistically aware measure of local curvature that respects
the rich structure of our partition complex.

\paragraph{Density Ricci Curvature}

The density component captures how data concentration varies along edges. In
smooth Riemannian geometry, Ricci curvature relates to the second derivative of
volume growth. Adapting this principle to our discrete setting, we measure the
convexity of density along an edge:
$$\text{Ric}_{\text{dens}}(v,w) = \frac{2}{\rho(e)} \left(\frac{\rho(v) + \rho(w)}{2} - \rho(e)\right)$$

The edge density $\rho(e)$ comes from our interpolation scheme developed in
Section 2.4.1, which incorporates geometric correction factors to respect the
Riemannian structure. When density is convex along the edge (the midpoint
density exceeds the average of endpoint densities), we obtain positive
curvature, suggesting smooth data distribution. Negative values indicate concave
density, potentially signaling a boundary cutting through a data cluster.

This measure could prove valuable for detecting problematic partition
boundaries. A decision boundary that slices through a dense cluster of similar
points will exhibit strongly negative density curvature, while boundaries that
respect natural data groupings show near-zero or positive values. This
diagnostic capability makes density curvature a powerful tool for understanding
model behavior.

\paragraph{Functional Ricci Curvature}

The functional component quantifies how model predictions transition between
cells. For the general case where models may have different functional forms in
adjacent cells, we need to measure the consistency of model behavior across
boundaries. We define different functional curvature measures to capture
different aspects of functional behavior across adjacent cells.

The mean curvature component measures consistency of local convexity:
$$\text{Ric}_{\text{func}}^{\text{mean}}(v,w) = 1 - \frac{|\kappa_f(v) - \kappa_f(w)|}{|\kappa_f(v)| + |\kappa_f(w)| + \epsilon}$$
where $\kappa_f$ is the discrete mean curvature from Section 2.4.2 and
$\epsilon > 0$ is a small regularization parameter that prevents division by
zero when both vertices have near-zero curvature. This captures whether the
function maintains consistent local behavior along the edge, with values close
to 1 indicating smooth transitions between regions of similar curvature and
negative values suggesting abrupt changes from convex to concave behavior.

The level set component examines how directional derivative variance changes:
$$\text{Ric}_{\text{func}}^{\text{level}}(v,w) = 1 - \frac{\sqrt{\kappa_f^{\text{level}}(v)} + \sqrt{\kappa_f^{\text{level}}(w)}}{2} \cdot \frac{|f(v) - f(w)|}{d_\rho(v,w) \cdot \|\nabla f\|_{\text{avg}}}$$
where
$\|\nabla f\|_{\text{avg}} = \frac{1}{|E|} \sum_{(u,u') \in E} \frac{|f(u) - f(u')|}{d_\rho(u,u')}$
represents the average gradient magnitude across all edges in the simplicial
complex. This normalization factor ensures the measure is scale-invariant and
combines local nonlinearity (through the level set curvatures) with the rate of
change along the edge relative to the global gradient behavior.

For models with continuous predictions across cells, we also include a direct
measure of gradient discontinuity:
$$\text{Ric}_{\text{func}}^{\text{direct}}(v,w) = 1 - \beta \frac{\int_{C_v \cap C_w} \|\nabla f_v - \nabla f_w\|^2 d\sigma}{\text{vol}_{n-1}(C_v \cap C_w) \cdot \|\nabla f\|_{\text{avg}}^2}$$
where $\beta > 0$ is a scaling parameter. Since the fraction measures the
squared gradient jump relative to the average squared gradient, it can range
from 0 (perfect continuity) to values exceeding 1 (opposing gradients). A
natural choice would be $\beta \in [0.25, 1]$ to ensure the functional Ricci
curvature remains appropriately bounded and comparable to the other components.
For piecewise constant models, this reduces to:
$$\text{Ric}_{\text{func}}^{\text{direct}}(v,w) = 1 - \gamma \frac{\|c_v - c_w\|^2}{\|c_v\|^2 + \|c_w\|^2}$$
where $\gamma > 0$ similarly provides scaling, with the choice of $\gamma$
depending on the desired sensitivity to prediction discontinuities across cell
boundaries.

The components above measure intrinsic properties of the function $f$. We can
also directly measure how well the model predictions align with the training
responses across the edge:
$$\text{Ric}_{\text{func}}^{\text{response}}(v,w) = 1 - \frac{\sum_{x_k \in C_v \cup C_w} |f(x_k) - y_k|^2}{\sum_{x_k \in C_v \cup C_w} |y_k - \bar{y}|^2}$$
where $\bar{y} = \frac{1}{|C_v \cup C_w|} \sum_{x_k \in C_v \cup C_w} y_k$ is
the mean response over the data points in the two cells. This measures
prediction accuracy relative to the variance in the data, with values near 1
indicating good local fit.

The functional Ricci curvature can be defined as a weighted combination:
$$\text{Ric}_{\text{func}}(v,w) = \alpha_1 \text{Ric}_{\text{func}}^{\text{mean}}(v,w) + \alpha_2 \text{Ric}_{\text{func}}^{\text{level}}(v,w) + \alpha_3 \text{Ric}_{\text{func}}^{\text{direct}}(v,w) + \alpha_4 \text{Ric}_{\text{func}}^{\text{response}}(v,w)$$
where the weights $\alpha_i$ represent a design choice in the regularization
framework. Different weight configurations emphasize different aspects: setting
only $\alpha_4 = 1$ focuses on prediction accuracy, while including other
components incorporates smoothness and consistency properties.

\paragraph{Regularization Through Geometric Penalties}

With the statistical Ricci curvature fully developed, we can now complete the regularization framework:
$$R(f) = \sum_{v \in V} \phi(\kappa^{\text{stat}}_f(v)) + \lambda \sum_{(v,w) \in E} \psi(\text{Ric}_{\text{stat}}(v,w))$$
The vertex term penalizes extreme local curvatures while allowing moderate
values that may represent legitimate data structure. The edge term penalizes
poor transitions between cells, encouraging models that respect both geometric
and statistical structure. The balance parameter $\lambda$ controls the relative
importance of local versus relational regularity.

The choice of penalty functions $\phi$ and $\psi$ depends on the specific
application and desired properties. For robust regularization that allows some
geometric variation while preventing extremes, one might choose:
$$\phi(x) = \begin{cases} x^2 & \text{if } |x| \leq \tau \\ 2\tau|x| - \tau^2 & \text{if } |x| > \tau \end{cases}$$
This Huber-like penalty transitions from quadratic to linear, preventing single
high-curvature vertices from dominating the regularization. For edges, a simpler
quadratic penalty $\psi(x) = x^{2}$ may suffice, though the relative stability and
informativeness of edge curvatures requires empirical validation.

\paragraph{Computational Considerations}

Computing statistical Ricci curvature requires several geometric quantities that
can be efficiently maintained during model training. The Wasserstein-1 distance
for geometric curvature can be computed using linear programming for small
vertex degrees or approximated using entropic regularization for larger
problems. The density interpolation requires solving the linear system from
Section 2.4.1 once per training epoch or when the partition structure changes
significantly.

For neural networks, the partition structure changes discretely when activation
patterns shift. Rather than recomputing all curvatures from scratch, we can
maintain incremental updates for affected edges. When a hyperplane shifts
slightly without changing the combinatorial structure, curvatures can be updated
through sensitivity analysis. This incremental approach makes geometric
regularization practical even for large models.

\paragraph{Higher-Order Terms}
The Taylor-like expansion in principle accommodates higher-order simplices
beyond edges. For triangles ($p=2$) and higher-dimensional simplices, one could
define analogous curvature measures that capture multi-way interactions among
cells. The diagnostic value and practical benefit of these higher-order terms
for machine learning applications remains an open question. We conjecture that
for many applications, the combination of vertex and edge terms may be
sufficient to capture the essential geometric structure.

\paragraph{Geometric Monitoring During Training}

The statistical curvature framework enables practical tools for monitoring and
improving model training. As parameters evolve during optimization, they induce
changes in the partition structure and corresponding geometric quantities:
$$\theta(t) \xrightarrow{\text{determines}} \mathcal{P}(t) \xrightarrow{\text{induces}} K(t) \xrightarrow{\text{equipped with}} \left(\langle \cdot, \cdot \rangle(t), \rho(t), \{f_\alpha(t)\}\right)$$
By tracking these geometric measures, we may gain insights complementary to
traditional loss-based monitoring.

A comprehensive geometric diagnostic tracks the distribution of curvatures
across the model:
$$\mathcal{C}(t) = \{\kappa^{\text{stat}}_f(v,t) : v \in V(t)\} \cup \{\text{Ric}_{\text{stat}}(v,w,t) : (v,w) \in E(t)\}$$
We hypothesize that early in training, highly variable curvatures with many
extreme values would be observed. As training progresses, the distribution might
stabilize if the model is learning appropriate geometric structure. A sudden
increase in negative curvatures could potentially provide an early warning
signal of overfitting.

The geometric energy provides a scalar summary of overall geometric
irregularity:
$$E(t) = \sum_{v \in V(t)} (\kappa^{\text{stat}}_f(v,t))^2 \cdot \text{vol}_n(C_v) + \lambda \sum_{(v,w) \in E(t)} (\text{Ric}_{\text{stat}}(v,w,t))^2 \cdot \text{vol}_{n-1}(F_{vw})$$
Monitoring this energy might reveal phases of learning: rapid decrease during
initial structure discovery, stabilization during fitting, and potential
increase if overfitting begins. Beyond passive monitoring, these geometric
insights could actively guide optimization.

For adaptive optimization, we propose adjusting learning rates based on how
strongly each parameter affects the geometric structure. If changing parameter
$\theta_i$ causes large changes in curvature values throughout the model, this
parameter might benefit from a smaller learning rate to avoid destabilizing the
geometry. Specifically, we could set:
$$\eta_i(t) = \eta_{\text{base}} \cdot g\left(\sum_{v \in V} \left|\frac{\partial \kappa^{\text{stat}}_f(v,t)}{\partial \theta_i}\right| + \sum_{(v,w) \in E} \left|\frac{\partial \text{Ric}_{\text{stat}}(v,w,t)}{\partial \theta_i}\right|\right)$$
where the sum measures the total sensitivity of all curvatures to parameter
$\theta_i$. The function $g$ is decreasing (for example, $g(x) = 1/(1+x)$), so
parameters with high geometric sensitivity get smaller learning rates. This
could help prevent training instabilities caused by rapid geometric changes.

For ensemble methods like gradient boosting, the geometric perspective could
reveal how models sequentially refine partitions. Each new weak learner modifies
the combined partition, with co-occurrence frequencies evolving as:
$$K^{(m+1)}(C_i, C_j) = \frac{m \cdot K^{(m)}(C_i, C_j) + \mathbf{1}[\text{adjacent in tree } m+1]}{m+1}$$

Monitoring whether these frequencies converge (potentially indicating geometric
consensus) or diverge (possibly suggesting inconsistent structures) might
provide insight into ensemble behavior beyond prediction accuracy.

These geometric monitoring tools could complement traditional metrics,
potentially revealing when models transition from learning robust patterns to
memorizing noise, even before validation performance degrades.

\subsubsection*{2.4.4 Open Questions and Future Directions}

While the geometric monitoring tools developed above offer immediate practical
applications, they also point toward deeper questions about the nature of
learning in machine learning models.

The relationship between geometry and learning extends beyond monitoring to
fundamental questions about architecture and regularization. Different
architectures create distinct initial geometric structures that constrain how
partitions can evolve during training. Understanding these constraints might
inform architecture design, considering not only expressiveness but also
geometric properties that affect learning dynamics.

Existing regularization techniques like dropout, weight decay, and batch
normalization demonstrably improve generalization, but their geometric effects
remain unclear. These methods might implicitly modify the evolution of partition
geometry to avoid problematic configurations. Alternatively, one could consider
regularization that explicitly targets geometric pathologies identified through
curvature analysis.

The relationship between statistical Ricci curvature and generalization bounds
presents another open question. Classical learning theory relies on complexity
measures such as VC dimension \cite{vapnik1971uniform, vapnik1998statistical}
and Rademacher complexity \cite{bartlett2002rademacher, mohri2012foundations}.
Models with bounded statistical Ricci curvature throughout training might
achieve different generalization properties, offering an alternative perspective
to classical measures.

These questions invite investigation from both theoretical and empirical
perspectives. Validating geometric insights requires careful experimentation
across diverse architectures and tasks. The connections between discrete
geometry and machine learning suggested here remain largely conjectural,
awaiting systematic exploration.

\section*{3. The Geometry of Feed-forward Networks}

In the previous section, we saw how partition models naturally induce Riemannian
simplicial complexes. Given data points $x_1, x_2, \ldots, x_N \in \mathbb{R}^n$
and responses $y_1, y_2, \ldots, y_N \in \mathbb{R}$, the geometry of the
partition $\mathcal{P}$ of a partition model is encapsulated by a nerve
simplicial complex $K(\mathcal{P})$ equipped with a Riemannian structure
$\langle , \rangle_{\mathcal{P}}$. Can neural networks be understood through
this same geometric lens?

Feed-forward networks represent the simplest architecture in deep learning.
While their expressive power has been well-studied \cite{poole2016exponential,
  bianchini2014complexity}, in this section we show that they also give rise to
rich geometric structures of functionally enriched Riemannian simplicial
complexes: geometric objects that encode both the combinatorial structure of the
network's partitions and the affine transformations it applies within each
region. Moreover, while tree-based models create a single partition, neural
networks generate a sequence of increasingly refined partitions through their
layers.

Consider the fundamental challenge: a neural network transforms data through
multiple layers, each creating its own geometric structure. How do these
structures relate to each other, and can we track geometric information as it
flows through the network? To answer this question, we begin with the simplest
building block—individual neurons—and progressively build our understanding of
how layers compose these geometric structures.

A single perceptron provides our starting point. Given a bounded domain
$\mathcal{D} \subset \mathbb{R}^n$, a perceptron is a mapping
$$
    \rho: \mathcal{D} \rightarrow \mathbb{R}
$$
defined by
$$
\rho(\mathbf{x}) = \sigma(\mathbf{w}^T\mathbf{x} + b),
$$
where $\mathbf{w} \in \mathbb{R}^n$ is the weight vector, $b \in \mathbb{R}$ is
the bias, and $\sigma$ is an activation function. The hyperplane
$H = \{\mathbf{x} : \mathbf{w}^T\mathbf{x} + b = 0\}$ divides $\mathbb{R}^n$
into two half-spaces:
$$
H^+ = \{\mathbf{x} : \mathbf{w}^T\mathbf{x} + b > 0\},
$$
and
$$
H^- = \{\mathbf{x} : \mathbf{w}^T\mathbf{x} + b \leq 0\}.
$$

This division induces a partition $\mathcal{P} = \{C_1, C_2\}$ of our domain,
where $C_1 = H^+ \cap \mathcal{D}$ and $C_2 = H^- \cap \mathcal{D}$, and the set
of local prediction functions
$$
\mathcal{F} = \{f_1 = \rho|_{C_{1}}: C_1 \to \mathbb{R}, f_2 = \rho|_{C_{2}}: C_2 \to \mathbb{R}\}
$$
For the ReLU activation function, the neuron outputs zero on $C_2$ and is linear on $C_1$, creating a piecewise linear function. For smooth activations, the partition marks regions of different qualitative behavior. The nerve complex $\mathcal{N}(\mathcal{P})$ of this partition is elementary: two vertices $v_1, v_2$ corresponding to the two cells, connected by a single edge since the cells share the boundary $H \cap \mathcal{D}$. The Riemannian structure
$$
\langle ,  \rangle_{\mathcal{N}(\mathcal{P})} = \{\langle ,  \rangle_{v_{1}}, \langle ,  \rangle_{v_{2}}, \langle ,  \rangle_{v_{1}\wedge v_{2}}\}
$$
over $\mathcal{N}(\mathcal{P})$ has the structure
$$
\langle v_{i} , v_{i} \rangle_{v_{i}} = \text{vol}_{n}(C_{i}) \text{  for  } i = 1, 2.
$$
and
$$
\langle v_{1}\wedge v_{2} , v_{1}\wedge v_{2} \rangle_{v_{1}\wedge v_{2}} = \text{vol}_{n-1}(C_{1} \cap C_{2}) = \text{vol}_{n-1}(H \cap \mathcal{D})
$$
The triple
$$
(\mathcal{N}(\mathcal{P}), \langle ,  \rangle_{\mathcal{N}(\mathcal{P})}, \mathcal{F})
$$
is the enriched simplicial complex associated with the partition $\mathcal{P}$ and the local prediction functions $\mathcal{F}$.

When we combine $m$ perceptrons into a single layer, the geometric structure becomes richer. The mapping $\rho: \mathcal{D} \rightarrow \mathbb{R}^m$ defined by
$$
\rho(x) = (\sigma(w_1^T x + b_1), \ldots, \sigma(w_m^T x + b_m)),
$$
creates $m$ hyperplanes that partition $\mathcal{D}$ into cells \cite{stanley2004introduction}. Each cell corresponds to an activation pattern $\alpha \in \{0,1\}^m$, indicating which neurons are active. Formally, the cell associated with pattern $\alpha$ is
$$
C_\alpha = \bigcap_{j: \alpha_j = 1} H_j^+ \cap \bigcap_{j: \alpha_j = 0} H_j^- \cap \mathcal{D}.
$$

This construction reveals an interesting connection to existing partition
models. A layer of ReLU neurons can be viewed as an enrichment of Multivariate
Adaptive Regression Splines (MARS) \cite{friedman1991multivariate}, where the
axis-aligned hinge functions $(x^{(j)} - t)_+$ are replaced by arbitrarily
oriented hyperplanes with ReLU activations $(\mathbf{w}^T x + b)_+$. Similarly,
it generalizes oblique decision trees by replacing their discrete splits with
continuous activation functions.

To illustrate how multiple neurons create richer geometric structures, consider a single layer neural network of two neurons with two inputs
$$
\rho(x_1, x_2) = \begin{pmatrix} \sigma(2x_1 - x_2 - 0.5) \\ \sigma(-x_1 + 2x_2 - 0.5) \end{pmatrix}
$$
operating on $\mathcal{D} = [0,1]^2$. We present this example in detail to demonstrate both the partition structure and the induced Riemannian geometry.

The first neuron activates when $2x_1 - x_2 > 0.5$, creating the hyperplane
$$
H_1: 2x_1 - x_2 = 0.5.
$$
The second neuron activates when $-x_1 + 2x_2 > 0.5$, creating the hyperplane
$$
H_2: -x_1 + 2x_2 = 0.5.
$$
These two lines intersect at $(0.5, 0.5)$, creating a partition $\mathcal{P}$ of $\mathcal{D}$ with four cells:
\begin{itemize}[label=, topsep=-5pt, partopsep=0pt, parsep=0pt, itemsep=0pt, after=\vspace{0.2cm}]
  \item $C_{00} = \{(x_1, x_2) \in \mathcal{D} : 2x_1 - x_2 \leq 0.5 \text{ and } -x_1 + 2x_2 \leq 0.5\}$
  \item $C_{01} = \{(x_1, x_2) \in \mathcal{D} : 2x_1 - x_2 \leq 0.5 \text{ and } -x_1 + 2x_2 > 0.5\}$
  \item $C_{10} = \{(x_1, x_2) \in \mathcal{D} : 2x_1 - x_2 > 0.5 \text{ and } -x_1 + 2x_2 \leq 0.5\}$
  \item $C_{11} = \{(x_1, x_2) \in \mathcal{D} : 2x_1 - x_2 > 0.5 \text{ and } -x_1 + 2x_2 > 0.5\}$
\end{itemize}
corresponding to the activation patterns, where we use subscript $ij$ as shorthand for the pattern $(i,j)$.
\begin{center}
  \begin{narrow}{-0.3in}{0in}
    \begin{minipage}[t]{.48\linewidth}
      \includegraphics[scale=0.55]{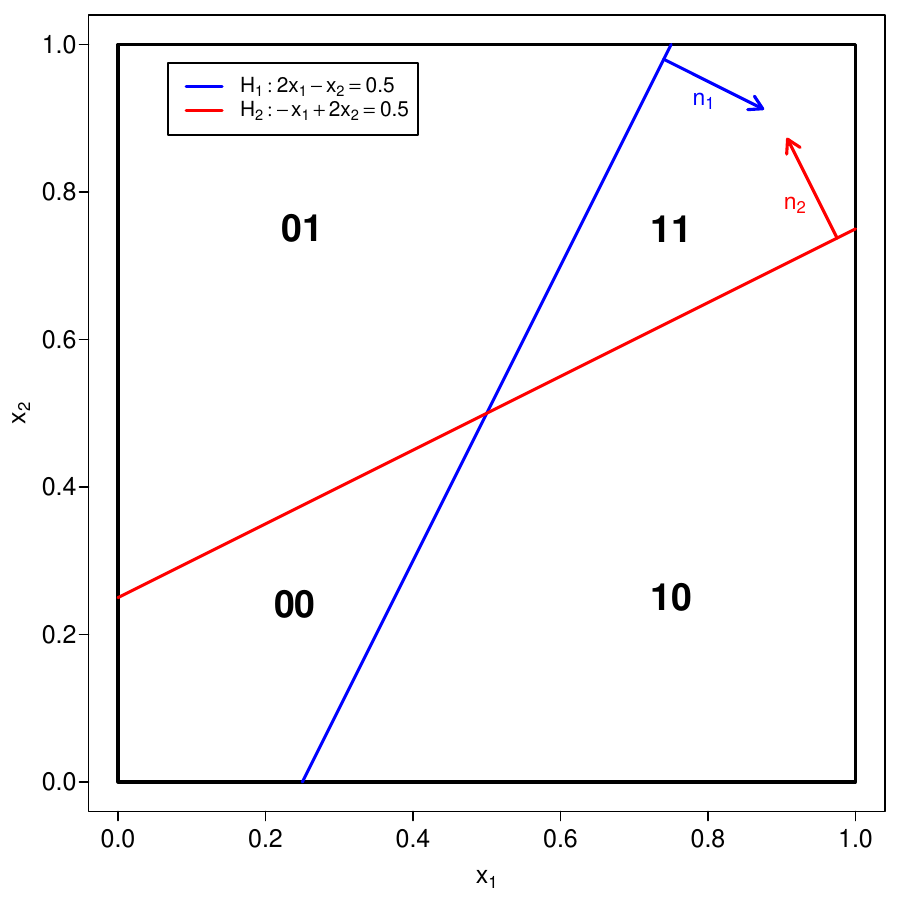}
    \end{minipage}\hfill\hspace{1cm}
    \begin{minipage}[t]{.48\linewidth}
      \includegraphics[scale=0.3]{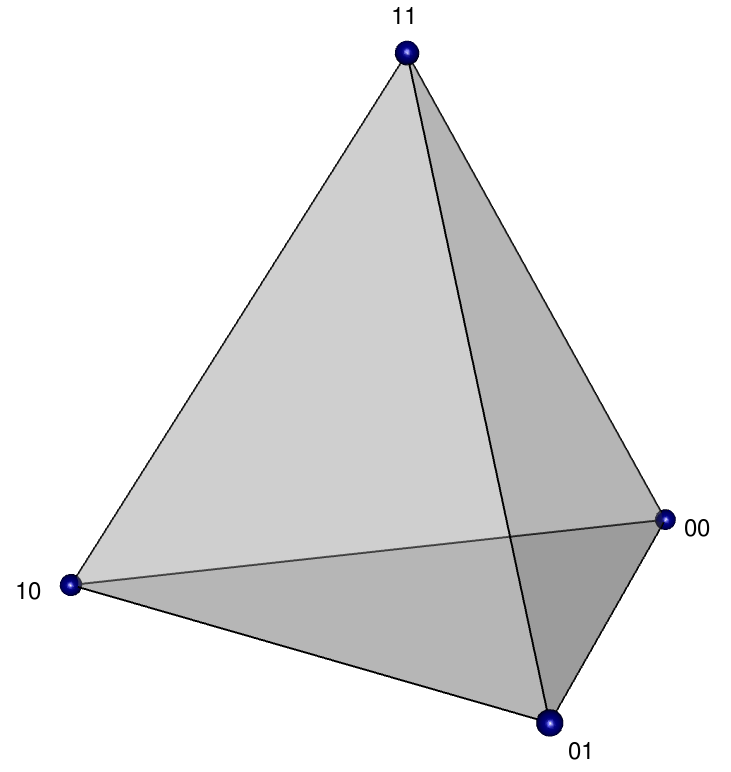}
    \end{minipage}
  \end{narrow}
  \captionof{figure}{Partition $\mathcal{P}$ induced by a two-neuron layer on
    the unit square. Left: The hyperplanes $H_1$ (blue) and $H_2$ (red) divide
    the domain into four cells $C_{00}$, $C_{01}$, $C_{10}$, and $C_{11}$
    corresponding to different activation patterns. Right: The nerve complex
    $\mathcal{N}(\mathcal{P})$ of this partition forms a 3-simplex simplicial
    complex since all four cells have a non-empty intersection.}
  \label{fig:two_neuron_partition}
\end{center}
The nerve complex $\mathcal{N}(\mathcal{P})$ of the partition $\mathcal{P}$ forms a 3-simplex simplicial complex since all four cells have a non-empty intersection. In particular, the simplicial complex $\mathcal{N}(\mathcal{P})$ has the vertices: $\{v_{00}, v_{01}, v_{10}, v_{11}\}$ corresponding to the cells $\{C_{00}, C_{01}, C_{10}, C_{11}\}$ and it has six edges
$$e_1 = v_{00} \wedge v_{01},\; e_2 = v_{00} \wedge v_{10},\; e_3 = v_{01} \wedge v_{11}$$
$$e_4 = v_{10} \wedge v_{11},\; e_5 = v_{00} \wedge v_{11},\; e_6 = v_{01} \wedge v_{10}$$
corresponding to faces
$$F_{1} = C_{00} \cap C_{01} =  H_1 \cap \{x_2 <= 0.5\}$$
$$F_{2} = C_{00} \cap C_{10} = H_{2} \cap \{x_1 <= 0.5\}$$
$$F_3   = C_{01} \cap C_{11} = H_2 \cap \{x_1 >= 0.5\}$$
$$F_4   = C_{10} \cap C_{11} = H_1 \cap \{x_2 >= 0.5\}$$
$$F_5   = C_{00} \cap C_{11} = \{(0.5, 0.5)\}$$
$$F_6   = C_{01} \cap C_{10} = \{(0.5, 0.5)\}.$$
Moreover, it has four 2D simplicies corresponding to the four triples of four
cells and it has one 3D simplex corresponding to the intersection
$\{(0.5, 0.5)\}$ of the four cells.

For each pair of edges with 1D face (thus for edges $e_{1}$ to $e_{4}$) the edge
inner product is
$$\langle e_i, e_j \rangle = \sqrt{\text{vol}_1(F_i) \cdot \text{vol}_1(F_j)} \cdot \cos\theta_{ij}$$
where $\text{vol}_1(F_i), \text{vol}_1(F_j)$ are the lengths of $F_{i}$ and $F_{j}$ faces and $\theta_{ij}$ is the interior dihedral angle between $F_i$ and $F_j$ within the corresponding cell. For example, the edges $e_1 = v_{00} \wedge v_{01}$ and $e_2 = v_{00} \wedge v_{10}$ correspond to faces $F_{1}$ and $F_{2}$ with the inner product
$$\langle e_1, e_2 \rangle_{v_{00}} = \sqrt{\text{vol}_1(F_1) \cdot \text{vol}_1(F_2)} \cdot \cos\theta_{12}$$
where
$$\cos \theta_{12} = \frac{\mathbf{n}_1 \cdot \mathbf{n}_2}{|\mathbf{n}_1||\mathbf{n}_2|} = \frac{4}{5}$$
with the normal vectors $\mathbf{n}_1 = (2, -1)$ and $\mathbf{n}_2 = (1, -2)$
chosen so that $\mathbf{n}_1$ faces away from the cell $C_{00}$ and
$\mathbf{n}_2$ faces into it. This non-zero off-diagonal term creates a
non-diagonal Gram matrix at vertex $v_{00}$, demonstrating that even a simple
two-neuron feed-forward network exhibits genuinely non-trivial Riemannian
structure.

This example illustrates how neural networks create geometric structures that
connect to classical partition-based methods in machine learning. The
hyperplanes divide the domain into regions with distinct computational
behaviors, while the nerve complex and its Riemannian structure capture both the
combinatorial relationships and the geometric configurations of these regions.

The true complexity emerges when we stack multiple layers. A feed-forward neural
network with $L$ layers defines a sequence of transformations:
$$\mathcal{D}_0 \xrightarrow{\rho_1} \mathcal{D}_1 \xrightarrow{\rho_2} \cdots \xrightarrow{\rho_{L-1}} \mathcal{D}_{L-1} \xrightarrow{\rho_L} \mathbb{R}^m,$$
where $\mathcal{D}_0 = \mathcal{D}$ is our initial domain and
$\mathcal{D}_\ell = \rho_\ell(\mathcal{D}_{\ell-1})$ for each subsequent layer.
Each domain $\mathcal{D}_\ell$ carries a partition $\mathcal{P}_\ell$ induced by
the hyperplanes of layer $\ell + 1$, giving rise to a sequence of nerve
complexes
$$
\mathcal{N}(\mathcal{P}_{0}), \mathcal{N}(\mathcal{P}_{1}), \ldots, \mathcal{N}(\mathcal{P}_{L-1}).
$$

Ideally, we would like each layer mapping $\rho_\ell$ to induce a well-defined
map between the nerve complexes that preserves their structure. Such a map
should send each vertex (cell) in the source complex to a vertex (cell) in the
target complex, and preserve adjacency relationships: if two cells are adjacent
in layer $\ell-1$, their images should be adjacent or coincide in layer $\ell$.
More generally, if a collection of cells forms a simplex (all cells have a
common intersection), their images should also form a simplex. Formally, a
simplicial map $f: K \to L$ between simplicial complexes is a function on
vertices $f: K_0 \to L_0$ such that whenever $\{v_0, v_1, ..., v_k\} \in K$
forms a $k$-simplex, the image vertices $\{f(v_0), f(v_1), ..., f(v_k)\}$ span a
simplex in $L$ (possibly of lower dimension if some vertices have the same
image).

However, the mapping
$\rho_\ell: \mathcal{D}_{\ell-1} \rightarrow \mathcal{D}_\ell$ does not
naturally induce such a simplicial map between the nerve complexes, because a
cell $C^{\ell-1}_\alpha$ from partition $\mathcal{P}_{\ell-1}$ maps under
$\rho_\ell$ to a set that may intersect multiple cells of partition
$\mathcal{P}_\ell$. There is no well-defined way to assign $C^{\ell-1}_\alpha$
to a single cell in the target partition.

To resolve this issue, we refine the partition $\mathcal{P}_{\ell-1}$ by
subdividing each cell according to how it maps into cells of $\mathcal{P}_\ell$.
For each pair of cells $(C^{\ell-1}_\alpha, C^\ell_\beta)$ such that
$\rho_\ell(C^{\ell-1}_\alpha) \cap C^\ell_\beta \neq \emptyset$, we define the
refined cell
$$Q^{\ell-1}_{\alpha,\beta} = C^{\ell-1}_\alpha \cap \rho_\ell^{-1}(C^\ell_\beta).$$
Since $\rho_\ell$ is piecewise affine, these refined cells are polyhedra that
partition the domain $\mathcal{D}_{\ell-1}$. The collection of all such refined
cells forms a partition $\mathcal{P}^{\rho_\ell}_{\ell-1}$ of
$\mathcal{D}_{\ell-1}$. By construction, each refined cell
$Q^{\ell-1}_{\alpha,\beta}$ maps under $\rho_\ell$ entirely into a single cell
$C^\ell_\beta$ of $\mathcal{P}_\ell$. This induces a simplicial map
$$
\hat{\rho}_\ell: K^{(\ell-1)}_{\rho_\ell} \to K^{(\ell)}
$$
sending each vertex corresponding to $Q^{\ell-1}_{\alpha,\beta}$ to the vertex
corresponding to $C^\ell_\beta$, where
$K^{(\ell-1)}_{\rho_\ell} = \mathcal{N}(\mathcal{P}^{\rho_\ell}_{\ell-1})$ and
$K^{(\ell)} = \mathcal{N}(\mathcal{P}_\ell)$ (see Appendix~B for the complete
construction and verification).

The refined simplicial complex $K^{(\ell-1)}_{\rho_\ell}$ has vertices
$v^{\ell-1}_{\alpha,\beta}$ corresponding to polytopes
$Q^{\ell-1}_{\alpha,\beta}$. The simplicial complex structure of
$K^{(\ell-1)}_{\rho_\ell}$ emerges naturally through the nerve of the the
partition $\mathcal{P}^{\rho_\ell}_{\ell-1}$ construction: vertices connect by
edges when their corresponding polytopes share a facet, triangles form when
three polytopes meet along an edge, and higher-dimensional simplices arise from
polytopes meeting at lower-dimensional faces. This construction captures the
full combinatorial complexity of how the refined partition cells fit together.

Crucially, this refinement preserves functional information. Over each refined
polytope $Q^{\ell-1}_{\alpha,\beta}$, the layer $\ell$ implements an affine map
$f^{\ell-1}_{\alpha,\beta}(x) = W^{\ell-1}_{\alpha,\beta} x + b^{\ell-1}_{\alpha,\beta}$,
where $W^{\ell-1}_{\alpha,\beta}$ is the weight matrix $W_\ell$ with rows
corresponding to inactive neurons (according to pattern $\alpha$) set to zero.
This functional enrichment is fundamental: each vertex in our simplicial complex
carries not just cell adjacency information but also the precise affine
transformation applied to data in that region.

Having established the existence of refined partitions and their associated
simplicial complexes, we now turn to the practical question of computing
their Riemannian structure.

This geometric structure can be systematically computed using differential forms
and pullback operations. For each cell $C^\ell_\beta$ in layer $\ell$, we
associate the $n$-form $\omega_\beta$ defined by integration over the cell. The
volume of the refined cell is then:
$$\text{vol}_n(Q^{\ell-1}_{\alpha,\beta}) = \left|\int_{C^{\ell-1}_\alpha} \rho_\ell^* \omega_\beta\right|$$
Since $\rho_\ell$ is affine on $C^{\ell-1}_\alpha$ with constant Jacobian matrix $W_\alpha$, this simplifies to:
$$\text{vol}_n(Q^{\ell-1}_{\alpha,\beta}) = |\det(W_\alpha)| \cdot \text{vol}_n(W_\alpha^{-1}(C^\ell_\beta - b_\alpha) \cap C^{\ell-1}_\alpha)$$

This pullback approach extends naturally to higher-dimensional chains. For edges
in the refined complex, boundary volumes are computed by pulling back the
appropriate $(n-1)$-forms, providing a systematic method to build the complete
Riemannian structure.

Given a finite dataset $X = \{x_1, \ldots, x_N\} \subset \mathcal{D}$ with
associated responses $y: X \to \mathbb{R}^m$, the feed-forward network learns an
approximation to the conditional expectation of $y$ given $X$. To properly track
geometric information through the network while maintaining simplicial maps, we
must work backwards from the final layer. This backward construction is
necessary because the composition $\rho_L \circ \cdots \circ \rho_1$ determines
which cells in the input space map to the same final output region. By starting
with the final partition and progressively refining earlier layers to respect
these identifications, we ensure that all intermediate simplicial maps are
well-defined.

This yields the sequence of simplicial complexes and maps:
$$K^{(0)}_{\rho_{L}, \ldots, \rho_1} \xrightarrow{\hat{\rho}_1} K^{(1)}_{\rho_{L}, \ldots, \rho_2} \xrightarrow{\hat{\rho}_2} \cdots \xrightarrow{\hat{\rho}_{L-1}} K^{(L-1)}_{\rho_L} \xrightarrow{\hat{\rho}_L} K^{(L)}$$
where each $K^{(\ell)}_{\rho_{L}, \ldots, \rho_{\ell+1}}$ is the nerve complex
of the data-containing cells in the partition of $\mathcal{D}_\ell$ induced by
the composition $\rho_L \circ \cdots \circ \rho_{\ell+1}$. Specifically, we only
include cells that contain the image under $\rho_\ell \circ \cdots \circ \rho_1$
of at least one data point from $X$. The crucial observation is that each
complex in this sequence is sufficiently refined to ensure that
$\hat{\rho}_\ell$ is a well-defined simplicial map.

The enriched Riemannian simplicial complex structure associated with the
feed-forward network consists of this entire sequence, not just the initial
complex. While the first element
$$\left(K^{(0)}_{\rho_{L}, \ldots, \rho_2, \rho_1}, \langle ,  \rangle_{K^{(0)}_{\rho_{L}, \ldots, \rho_2, \rho_1}}, \mathcal{F}^{(0)}_{\rho_{L}, \ldots, \rho_2, \rho_1}\right)$$
captures the final geometric structure induced on the data-supporting region of
the input domain, the complete characterization requires the entire sequence of
complexes together with the simplicial maps between them. Here,
$\langle , \rangle_{K^{(0)}_{\rho_{L}, \ldots, \rho_2, \rho_1}}$ is the
Riemannian structure derived from the data-containing cells of the partition of
$\mathcal{D}$, and $\mathcal{F}^{(0)}_{\rho_{L}, \ldots, \rho_2, \rho_1}$
consists of the restrictions of the full network function
$\rho_L \circ \cdots \circ \rho_1$ to each such cell.

This sequential perspective reveals the true richness of the neural network's
geometric structure. Each simplicial map $\hat{\rho}_\ell$ tracks how geometric
information propagates through layer $\ell$, while the sequence as a whole
captures how the network progressively refines its partition of the input space.
This allows us to define invariants that utilize the entire construction
process, not just the final result. For instance, we can track how the
dimensions of simplices evolve through the sequence, measure the distortion
introduced by each layer map, or analyze how the Riemannian structure changes as
information flows through the network.

The differential forms framework provides a systematic approach to computing the
Riemannian structure. For the 0-chain structure, each vertex $v_{\alpha,\beta}$
in the refined complex has inner product:
$$\langle v_{\alpha,\beta}, v_{\alpha,\beta} \rangle_0 = \text{vol}_n(Q^{\ell-1}_{\alpha,\beta})$$
computed via the pullback formula above. For 1-chains, consider edges
$e = v_{\alpha,\beta} \wedge v_{\alpha',\beta'}$. The boundary
$F = \overline{Q^{\ell-1}_{\alpha,\beta}} \cap \overline{Q^{\ell-1}_{\alpha',\beta'}}$
has $(n-1)$-dimensional volume that can be computed by pulling back the appropriate
boundary forms.

This approach extends naturally to compositions of layers. For a deep network
$\rho_L \circ \cdots \circ \rho_1$, we work backwards from the final partition,
computing volumes via composed pullbacks:
$$\text{vol}_n(Q^{(0)}) = \left|\int_{C^{(0)}} (\rho_1^* \circ \rho_2^* \circ \cdots \circ \rho_L^*) \omega_{\text{final}}\right|$$

The differential forms framework provides a path forward for computing this
Riemannian structure. Rather than viewing the refined cells as intractable
geometric objects, we can leverage the fact that volumes and boundary measures
transform systematically under pullback. For practical implementation, we need
only track cells containing training data and compute their volumes via composed
pullbacks. The geometric regularization methods developed in Section 2 can then
be applied using these computed volumes and boundary measures, potentially
leading to networks with more interpretable decision boundaries, better
generalization properties, and improved robustness.

This geometric perspective transforms neural networks from opaque function
approximators into mathematically precise geometric objects. The key insight is
that neural networks implement functionally enriched structures where each layer
creates both a simplicial complex encoding partition topology and a collection
of affine maps specifying local computations. The differential forms approach
makes this dual structure computationally accessible: the pullback operation
systematically tracks how volumes, angles, and other geometric quantities
transform through the network.

The framework extends naturally beyond ReLU networks. For smooth activation
functions like sigmoid or tanh, the layer maps $\rho_\ell$ are smooth rather
than piecewise affine. The pullback of forms remains well-defined:
$$
(\rho_\ell^* \omega)(x) = \omega(\rho_\ell(x)) \cdot \det(J_{\rho_\ell}(x)),
$$
where $J_{\rho_\ell}(x)$ is the Jacobian matrix at point $x$. While smooth
activations don't create discrete partitions, we can use data-induced partitions
(such as Voronoi cells from training points) and apply the pullback formalism to
compute induced volumes. The varying Jacobian creates a naturally curved
Riemannian geometry that captures how the network distorts space continuously
rather than through discrete cells. For attention mechanisms in Transformers,
the data-dependent partitions created by attention patterns fit naturally into
this framework, with attention weights providing additional geometric structure.

In practice, several strategies make implementation feasible. We need only track
cells containing training data, reducing complexity from potentially exponential
in the number of neurons to linear in the number of training examples. During
forward propagation, recording activation patterns identifies which refined
cells actually influence model behavior. For ReLU networks with affine maps
$\rho_\ell(x) = W_\alpha x + b_\alpha$ on each cell, the pullback formulas
provide explicit volume computations through the Jacobian determinant
$|\det(W_\alpha)|$.

The differential forms framework also suggests new theoretical questions
bridging abstract mathematics and practical machine learning. Can we
characterize which architectures maintain tractable geometric structures under
pullback? How does the condition number of Jacobian matrices relate to
generalization? What geometric invariants predict model robustness? These
questions open avenues for both theoretical advances and practical tools for
analyzing neural networks.

While challenges remain in scaling these computations to very large networks,
the systematic nature of the pullback formalism provides a principled path
forward. The intersection of discrete geometry with deep learning offers rich
opportunities for developing new regularization methods, interpretability tools,
and architectural insights grounded in rigorous mathematical foundations.

\section*{Appendix A: Derivation of the Whitney Extension Formula}
\addcontentsline{toc}{section}{Appendix A: Derivation of the Whitney Extension Formula}

We establish that both the Whitney extension and harmonic extension of a vertex
function to p-simplices yield the same formula, namely the barycentric average
weighted by simplex volume. Consider a p-simplex
$\sigma = [v_0, v_1, \ldots, v_p]$ in a Riemannian simplicial complex $K$. The
Riemannian structure determines edge lengths $\ell_{ij} = d(v_i, v_j)$ for all
vertex pairs, where the distance function $d$ is induced by the Riemannian
metric on the 1-skeleton of $K$. These edge lengths allow us to embed $\sigma$
isometrically into Euclidean space $\mathbb{R}^{p+1}$ using the Cayley-Menger
construction \cite{blumenthal1953theory, berger2009geometry}, establishing a
concrete geometric realization that preserves all distance relationships.

For the embedded simplex, we introduce barycentric coordinates
$(\lambda_0, \lambda_1, \ldots, \lambda_p)$ such that any point $x \in \sigma$
can be uniquely expressed as $x = \sum_{i=0}^p \lambda_i v_i$ where
$\lambda_i \geq 0$ and $\sum_{i=0}^p \lambda_i = 1$. These coordinates provide a
natural parametrization of the simplex that respects its affine structure
\cite{coxeter1969introduction}. The constraint $\sum_{i=0}^p \lambda_i = 1$
implies that only $p$ of these coordinates are independent; we may express
$\lambda_0 = 1 - \sum_{j=1}^p \lambda_j$ and use
$(\lambda_1, \ldots, \lambda_p)$ as free parameters ranging over the standard
p-simplex in parameter space.

The Whitney extension formula for lifting a vertex function
$u: V(K) \rightarrow \mathbb{R}$ to a p-cochain is given by
\cite{whitney1957geometric}
$$\kappa_p^{\text{Whitney}}(u)(\sigma) = \sum_{i=0}^p u(v_i) \int_\sigma \lambda_i \, d\lambda_0 \wedge \cdots \wedge \widehat{d\lambda_i} \wedge \cdots \wedge d\lambda_p$$
where the hat indicates omission of the corresponding differential. To evaluate
these integrals, we first observe that due to the linear constraint among
barycentric coordinates, we have $d\lambda_0 = -\sum_{j=1}^p d\lambda_j$. This
relationship allows us to express any differential form involving $d\lambda_0$
in terms of the independent differentials $d\lambda_1, \ldots, d\lambda_p$.

Consider the integral
$\int_\sigma \lambda_i \, d\lambda_0 \wedge \cdots \wedge \widehat{d\lambda_i} \wedge \cdots \wedge d\lambda_p$.
When $i = 0$, we have
$$\int_\sigma \lambda_0 \, d\lambda_1 \wedge \cdots \wedge d\lambda_p = \int_\sigma \left(1 - \sum_{j=1}^p \lambda_j\right) d\lambda_1 \wedge \cdots \wedge d\lambda_p$$

For $i \geq 1$, substituting $d\lambda_0 = -\sum_{j=1}^p d\lambda_j$ into the
differential form and using the antisymmetry of the wedge product, we find that
$$d\lambda_0 \wedge d\lambda_1 \wedge \cdots \wedge \widehat{d\lambda_i} \wedge \cdots \wedge d\lambda_p = (-1)^i d\lambda_1 \wedge \cdots \wedge d\lambda_p$$

Thus all integrals reduce to computing $\int_\sigma \lambda_i \, dV_\sigma$,
where $dV_\sigma$ represents the volume element on the p-dimensional simplex. By
the symmetry of the simplex under permutations of vertices, each barycentric
coordinate has the same integral value. Since $\sum_{i=0}^p \lambda_i = 1$
identically on $\sigma$, we have
$$(p+1) \int_\sigma \lambda_i \, dV_\sigma = \int_\sigma \sum_{i=0}^p \lambda_i \, dV_\sigma = \int_\sigma 1 \, dV_\sigma = \text{vol}_p(\sigma)$$

Therefore, $\int_\sigma \lambda_i \, dV_\sigma = \text{vol}_p(\sigma)/(p+1)$ for
each $i$. This classical result about the integral of barycentric coordinates
over a simplex can be found in \cite{lasserre1999integration} and follows from
the fact that the barycenter divides the simplex into $p+1$ equal-volume parts.
Thus we obtain
$$\kappa_p^{\text{Whitney}}(u)(\sigma) = \sum_{i=0}^p u(v_i) \cdot \frac{\text{vol}_p(\sigma)}{p+1} = \frac{\text{vol}_p(\sigma)}{p+1} \sum_{i=0}^p u(v_i)$$

For the harmonic extension approach, we seek a function
$H: \sigma \rightarrow \mathbb{R}$ that minimizes the Dirichlet energy
$\int_\sigma \|\nabla H\|^2 \, dV_\sigma$ subject to the boundary conditions
$H(v_i) = u(v_i)$ for all vertices. The Euler-Lagrange equation for this
variational problem is Laplace's equation $\Delta H = 0$ on the interior of
$\sigma$ \cite{evans2010partial}. Since $\sigma$ is a simplex, the solution is
the unique affine function interpolating the vertex values, namely
$H(x) = \sum_{i=0}^p \lambda_i(x) u(v_i)$. This can be verified directly: the
barycentric coordinates $\lambda_i$ are affine functions of the Cartesian
coordinates, hence their second derivatives vanish, giving
$\Delta H = \sum_{i=0}^p u(v_i) \Delta \lambda_i = 0$.

The p-cochain value is then obtained by integrating this harmonic function:
$$\kappa_p^{\text{harm}}(u)(\sigma) = \int_\sigma H \, dV_\sigma = \int_\sigma \sum_{i=0}^p \lambda_i u(v_i) \, dV_\sigma = \sum_{i=0}^p u(v_i) \int_\sigma \lambda_i \, dV_\sigma$$

Using our previous calculation, this yields
$$\kappa_p^{\text{harm}}(u)(\sigma) = \frac{\text{vol}_p(\sigma)}{p+1} \sum_{i=0}^p u(v_i)$$

Thus both the Whitney extension and harmonic extension produce the same lifted
values, demonstrating that the formula represents a fundamental geometric
property of simplicial interpolation. The volume $\text{vol}_p(\sigma)$
appearing in this formula is determined by the Riemannian structure through the
edge lengths. For a p-simplex with edge lengths
$\{\ell_{ij}\}_{0 \leq i < j \leq p}$, the volume can be computed using the
Cayley-Menger determinant \cite{berger2009geometry}:
$$\text{vol}_p(\sigma)^2 = \frac{(-1)^{p+1}}{2^p (p!)^2} \begin{vmatrix}
0 & 1 & 1 & \cdots & 1 \\
1 & 0 & \ell_{01}^2 & \cdots & \ell_{0p}^2 \\
1 & \ell_{01}^2 & 0 & \cdots & \ell_{1p}^2 \\
\vdots & \vdots & \vdots & \ddots & \vdots \\
1 & \ell_{0p}^2 & \ell_{1p}^2 & \cdots & 0
\end{vmatrix}$$

This completes the derivation, showing that the lifted value on any p-simplex
equals the volume-weighted barycentric average of the vertex values, with the
volume determined by the Riemannian structure of the simplicial complex.

\section*{Appendix B: Construction of Simplicial Maps Between Neural Network Layers}
\addcontentsline{toc}{section}{Appendix B: Construction of Simplicial Maps Between Neural Network Layers}

This appendix provides the rigorous construction of simplicial maps between
consecutive layers of a neural network, complementing the differential forms
framework introduced in Section 3. While Section 3 showed how to compute the
Riemannian structure via pullback operations, here we establish the foundational
result that well-defined simplicial maps exist between the nerve complexes of
consecutive layers.

Consider a feed-forward neural network with layer mappings
$\rho_\ell: \mathbb{R}^{n_{\ell-1}} \rightarrow \mathbb{R}^{n_\ell}$. As
established in Section 2.2, each ReLU layer induces a polyhedral partition of
its input space. The challenge is that when a cell from layer $\ell-1$ maps
forward under $\rho_\ell$, its image typically intersects multiple cells in
layer $\ell$'s partition, so $\rho_\ell$ does not directly induce a simplicial
map between the nerve complexes.

We resolve this through a refinement procedure. Given partitions $\mathcal{P}_{\ell-1} = \{C^{\ell-1}_\alpha\}$ and $\mathcal{P}_\ell = \{C^\ell_\beta\}$, we define the refined partition consisting of cells:
$$Q^{\ell-1}_{\alpha,\beta} = C^{\ell-1}_\alpha \cap \rho_\ell^{-1}(C^\ell_\beta)$$

\begin{lemma}
\label{lem:refined_partition}
The collection $\mathcal{Q}_{\ell-1} = \{Q^{\ell-1}_{\alpha,\beta} : Q^{\ell-1}_{\alpha,\beta} \neq \emptyset\}$ forms a partition of $\mathcal{D}_{\ell-1}$.
\end{lemma}

\begin{proof}
We verify coverage and disjointness. For coverage, any $x \in \mathcal{D}_{\ell-1}$ belongs to some $C^{\ell-1}_\alpha$ and $\rho_\ell(x)$ belongs to some $C^\ell_\beta$, hence $x \in Q^{\ell-1}_{\alpha,\beta}$. For disjointness, if $x \in Q^{\ell-1}_{\alpha,\beta} \cap Q^{\ell-1}_{\alpha',\beta'}$, then $x \in C^{\ell-1}_\alpha \cap C^{\ell-1}_{\alpha'}$ implies $\alpha = \alpha'$, and $\rho_\ell(x) \in C^\ell_\beta \cap C^\ell_{\beta'}$ implies $\beta = \beta'$.
\end{proof}

Each refined cell remains a convex polyhedron since $\rho_\ell$ is affine on $C^{\ell-1}_\alpha$ with $\rho_\ell(x) = W_\alpha x + b_\alpha$. The volume of the refined cell, as shown in Section 3, is computed via the pullback formula:
$$\text{vol}_n(Q^{\ell-1}_{\alpha,\beta}) = |\det(W_\alpha)| \cdot \text{vol}_n(W_\alpha^{-1}(C^\ell_\beta - b_\alpha) \cap C^{\ell-1}_\alpha)$$

\begin{theorem}
\label{thm:simplicial_map}
The map $\hat{\rho}_\ell: K^{(\ell-1)}_{\text{ref}} \rightarrow K^{(\ell)}$ defined by $\hat{\rho}_\ell(v_{\alpha,\beta}) = v_\beta$ on vertices and extended linearly to simplices is a well-defined simplicial map.
\end{theorem}

\begin{proof}
Let $\sigma = [v_{\alpha_0,\beta_0}, \ldots, v_{\alpha_k,\beta_k}]$ be a simplex in $K^{(\ell-1)}_{\text{ref}}$. By definition, $\bigcap_{i=0}^k Q^{\ell-1}_{\alpha_i,\beta_i} \neq \emptyset$. For any $x$ in this intersection, we have $\rho_\ell(x) \in \bigcap_{i=0}^k C^\ell_{\beta_i}$, which is non-empty. Therefore the vertices $\{v_{\beta_i}\}$ span a simplex in $K^{(\ell)}$. Face relations are preserved since faces of $\sigma$ map to faces of $\hat{\rho}_\ell(\sigma)$.
\end{proof}

The construction satisfies functorial properties: identity maps induce identity simplicial maps, and the diagram for composed layer maps commutes after appropriate refinements.

For deep networks, the backward construction from Section 3.1 ensures compatibility across all layers. Starting from the final partition induced by $\rho_L \circ \cdots \circ \rho_1$, we progressively refine earlier partitions to obtain the sequence:
$$K^{(0)}_{\rho_{L}, \ldots, \rho_1} \xrightarrow{\hat{\rho}_1} K^{(1)}_{\rho_{L}, \ldots, \rho_2} \xrightarrow{\hat{\rho}_2} \cdots \xrightarrow{\hat{\rho}_L} K^{(L)}$$

The differential forms framework from Section 3 shows how to compute the Riemannian structure on these complexes through systematic pullback calculations. The Jacobian matrices $W_\alpha$ that appear in the volume formula above are precisely the local linear approximations that determine how geometric quantities transform under the layer mappings.

This completes the rigorous foundation for analyzing neural networks as sequences of simplicial complexes with well-defined maps between them, while the computational aspects are handled through the differential forms machinery developed in the main text.

\bibliographystyle{plain}
\bibliography{no3_paper}

\end{document}